\documentclass{article}
\usepackage{iclr2027_conference,times}

\usepackage{amsmath,amsfonts,bm}

\def\eqref#1{equation~\ref{#1}}

\def\1{\bm{1}}

\DeclareMathAlphabet{\mathsfit}{\encodingdefault}{\sfdefault}{m}{sl}
\SetMathAlphabet{\mathsfit}{bold}{\encodingdefault}{\sfdefault}{bx}{n}

\usepackage[utf8]{inputenc}
\usepackage[T1]{fontenc}
\usepackage{microtype}
\usepackage{graphicx}
\usepackage{subcaption}
\usepackage{booktabs}
\usepackage{diagbox}
\usepackage{tabularx}
\usepackage{multirow}
\usepackage{wrapfig}
\usepackage{dsfont}
\usepackage{listings}
\usepackage{url}
\usepackage{amsfonts}
\usepackage{nicefrac}
\usepackage{xcolor}
\usepackage{amsmath}
\usepackage{amssymb}
\usepackage{mathtools}
\usepackage{amsthm}
\usepackage{enumitem}
\usepackage{algorithm}
\usepackage{algorithmic}
\usepackage{float}
\usepackage{array}
\usepackage{placeins}
\usepackage{ulem}
\usepackage{bbm}
\usepackage{thmtools}
\usepackage{thm-restate}
\usepackage{hyperref}

\lstdefinestyle{xmlbox}{
  basicstyle=\ttfamily\footnotesize,
  breaklines=true,
  columns=fullflexible,
  keepspaces=true,
  frame=single,
  showstringspaces=false
}

\newcolumntype{S}[1]{>{\raggedright\arraybackslash\hspace{0pt}}p{#1}}
\newsavebox{\compEvalTableBox}

\theoremstyle{plain}
\newtheorem{theorem}{Theorem}[section]

\newtheorem{lemma}[theorem]{Lemma}

\theoremstyle{definition}

\theoremstyle{remark}

\newenvironment{restatetheorem}[1]{%
  \renewcommand{\theinnerthm}{#1}%
  \trivlist
  \item[\hskip \labelsep{\bfseries #1}]%
}{%
  \endtrivlist
}

\title{Generative Support Realignment for Cross-Domain Offline Reinforcement Learning}

\author{
  Minung Kim \quad\qquad Jeongmo Kim \quad\qquad Gwanwoo Choi \quad\qquad Seungyul Han\thanks{Corresponding Author}\\
  ~\\
  Ulsan National Institute of Science and Technology, UNIST\\
  \texttt{\{minungkim,jmkim22,cgw1999,syhan\}@unist.ac.kr}
}

\iclrpreprintcopy

\begin{document}

\maketitle

\vspace{-0.3in}
\begin{abstract}
\vspace{-0.1in}
 Cross-domain offline reinforcement learning learns a target policy from pre-collected source and target datasets with different dynamics. When target data are scarce, effectively compensating for their limited coverage using source data remains challenging due to the discrepancy between domains. We propose \textit{Target-aligned Coverage Expansion (TCE)}, which leverages source states to generatively realign and expand the limited target support while controlling the generation error induced by this expansion. We further derive a performance gap bound that characterizes the interplay between generation error and source--target dynamics gap, providing theoretical guidance for effective coverage expansion and source utilization. Across diverse cross-domain environments, TCE consistently outperforms state-of-the-art baselines, while further analysis empirically supports the guidance provided by our theoretical findings. \end{abstract}


\vspace{-0.1in}

\section{Introduction}
\label{sec:intro}

\vspace{-0.1in}


Cross-domain reinforcement learning (cross-domain RL) aims to transfer a policy learned in a source domain to a target domain with different environment dynamics. This setting arises in real-world applications such as heterogeneous robot control, simulation-to-real autonomous driving, and medical decision making \citep{healthcare,autonomousdriving}. A variety of RL approaches have been proposed to address such domain shifts \citep{darc,imit1}, but most assume online interaction with the source or target environment to collect additional data during training. In practice, such interaction can be costly, unsafe, or even infeasible \citep{offlineReview}. Cross-domain offline RL addresses this limitation by learning exclusively from pre-collected datasets from both domains \citep{dara}. This problem becomes particularly challenging when target data are scarce, making effective use of source data crucial for learning in the target domain.

Existing cross-domain offline RL methods primarily exploit source data in two ways. One line of work selects source transitions similar to the target domain through contrastive representation learning \citep{wen2024contrastive}. The resulting dataset can then be used to train policies with standard offline RL algorithms such as CQL \citep{cql} or IQL \citep{iql}. Recent methods further consider smaller target-data budgets by selectively incorporating target-near source transitions using optimal transport \citep{lyu2025cross}. Another line of work generates additional training data by leveraging source data \citep{dmc,mobody}. While these approaches improve the use of source information, effectively compensating for scarce target data remains challenging due to source--target dynamics gap and generation error.

This raises a central question: \textit{how can source data best compensate for limited target coverage while keeping the resulting performance gap small?} To address this question, we derive a theoretical performance bound that disentangles the dynamics gap from source transitions and the generation error from target-aligned generation. Based on this analysis, we propose \textbf{Target-aligned Coverage Expansion (TCE)}, which aims to expand the limited target coverage while keeping the resulting performance gap small. To this end, TCE exploits the broader state coverage of source data to expand the support of target-aligned generation while controlling the error induced by this expansion, and additionally incorporates target-near source transitions when favorable under the bound. Extensive experiments across diverse domain shifts demonstrate consistent improvements over existing methods, and further analysis supports the theoretical implications and identifies when direct source incorporation provides additional benefits.

\vspace{-0.1in}

{
\section{Related Work}
\label{sec:related}
\vspace{-0.1in}
\paragraph{Offline Reinforcement Learning.}
Offline RL learns policies entirely from pre-collected datasets without additional environment interaction. Representative methods such as CQL~\citep{cql} and IQL~\citep{iql} achieve strong performance on standard offline RL benchmarks~\citep{d4rl}, while subsequent studies further improve offline learning through advanced penalty and constraint mechanisms~\citep{lyu2022mildly,yeom2024exclusively}. Other approaches address distributional shift through behavior regularization~\citep{BRAC} or data filtering based on mutual information~\citep{vabound,tightMI}. More recently, generative approaches have been explored to augment limited offline datasets~\citep{li2024diffstitch,luo2025generative,chae2026flow}. However, learning from data collected under different environment dynamics introduces an additional challenge, motivating cross-domain offline RL.

\vspace{-0.1in}

\paragraph{Cross-Domain Reinforcement Learning.}
Cross-domain RL addresses substantial dynamics shifts across distinct domains, beyond the relatively small task variations considered in meta-learning~\citep{finn2017model,pmlr-v267-kim25ag,han2026self}, using approaches such as reward correction~\citep{darc} and value-guided data filtering~\citep{vgdf}. Cross-domain imitation learning similarly transfers demonstrations across domains~\citep{imit1,imit3,choi2023domain}, including adaptation to varying environment dynamics~\citep{imit2}. However, online approaches require target interaction~\citep{lyu2024cross}, limiting their applicability to purely offline settings. Offline cross-domain RL~\citep{wen2024contrastive} has therefore explored source selection~\citep{lyu2025cross}, transition editing~\citep{xTED}, robust adaptation~\citep{DROCO}, and data generation~\citep{dmc,mobody}. While these methods primarily mitigate dynamics mismatch or augment available data, TCE uses source states to expand limited target support and generates target-aligned transitions over the expanded support, controlling the expansion to limit the resulting performance gap.

\vspace{-0.1in}

\paragraph{Score-Based Models and Diffusion Processes.} Two approaches to score-based generative modeling have enabled high-quality sample generation: denoising score matching, which estimates gradients of the data log-density across multiple noise scales \citep{song2019generative, song2020improved}, and diffusion models, which progressively corrupt and then denoise data through a sequence of intermediate steps \citep{ho2020denoising, nichol2021improved, dhariwal2021diffusion}. The stochastic differential equations (SDEs) framework unifies both approaches and enables continuous-time sampling \citep{song2020score, karras2022elucidating}. In this paper, we leverage the SDE formalism for two complementary generation processes: expanding state coverage beyond the limited target support and generating target-aligned state pairs over the expanded states.
}

\vspace{-0.1in}
\section{Background}
\label{sec:background}

\vspace{-0.05in}
\subsection{Cross-Domain Offline RL Setup}
\vspace{-0.05in}

We define a Markov Decision Process (MDP) as $\mathcal{M}=(\mathcal{S},\mathcal{A},P_\mathcal{M},R,\gamma)$, where $\mathcal{S}$ is the state space, $\mathcal{A}$ the action space, $P_\mathcal{M}$ the transition dynamics, $R$ the reward function, and $\gamma$ the discount factor. In the cross-domain setting, we assume access to a source domain $\mathcal{M}_{\mathrm{src}}=(\mathcal{S},\mathcal{A},P_{\mathrm{src}},R,\gamma)$ and a target domain $\mathcal{M}_{\mathrm{tar}}=(\mathcal{S},\mathcal{A},P_{\mathrm{tar}},R,\gamma)$, which share the same state and action spaces as well as the reward function but differ in their transition dynamics, i.e., $P_{\mathrm{src}}\neq P_{\mathrm{tar}}$. In the cross-domain offline setting, the agent cannot interact with either domain and must rely solely on pre-collected transitions $(s_t,a_t,r_t,s_{t+1})$, where $s_t\in\mathcal{S}$ denotes the state, $a_t\in\mathcal{A}$ the action, $r_t=R(s_t,a_t)$ the reward, and $s_{t+1}\sim P(\cdot|s_t,a_t)$ the next state with $P=P_{\mathrm{src}}$ or $P=P_{\mathrm{tar}}$. We denote the datasets collected from the source and target domains as $\mathcal{D}_{\mathrm{src}}$ and $\mathcal{D}_{\mathrm{tar}}$, respectively, under the practical constraint that $|\mathcal{D}_{\mathrm{tar}}|\ll|\mathcal{D}_{\mathrm{src}}|$, making direct policy learning on the target domain challenging.

\vspace{-0.05in}
\subsection{Score-based Generative Models with SDEs}
\vspace{-0.05in}
\label{subsec:sdes}
Generative models aim to learn the data distribution $p_{\text{data}}(x)$ and generate realistic samples. Representative approaches include Generative Adversarial Networks (GANs) \citep{goodfellow2014generative}, Variational Auto-Encoders (VAEs) \citep{kingma2013auto}, and diffusion models \citep{ho2020denoising}. Among these, score-based generative models formulated with stochastic differential equations (SDEs) \citep{song2020score} provide a continuous-time view of diffusion, support flexible noise schedules, and enable efficient sampling and likelihood computation via the probability-flow ODE. Given a clean sample $x^0$, we perturb it with Gaussian noise as $x^{\tau} = x^0 + \sigma(\tau)z$, where $z \sim \mathcal{N}(0,I)$, $\tau \in [0,1]$ denotes the continuous noise level, and $\sigma(\tau)$ is the noise scale. In this paper, we consider a conditional input $c$ and train a score network $q_\theta(x,\tau \mid c)$ to estimate the conditional score $\nabla_x \log p_{\tau}(x \mid c)$ via denoising score matching:
{
{\small
\begin{equation}
\mathcal{L}_{\mathrm{score}}(\theta) = \mathbb{E}_{\tau,(x^0,c)}\!\left[ \zeta(\tau)\,\left\|\,q_\theta(x^{\tau},\tau \mid c) + z/\sigma(\tau)\right\|_2^{2} \right],
\label{eq:baseloss}
\end{equation}
}
where $\zeta(\tau)$ is a time-dependent weight. At sampling time, starting from $x^1 \sim \mathcal{N}(0,\sigma (1)^2 I)$, samples are generated by solving the discretized reverse SDE using the Predictor--Corrector sampler with discretized noise levels $\tau^k$ ($1=\tau^K > \cdots > \tau^0=0$):
{\small
\begin{equation}
\label{eq:basesampling}
\resizebox{0.90\columnwidth}{!}{$
x^{k-1} = x^k + \big[f(x^k,\tau^k) - g(\tau^k)^{2}q_\theta(x^k,\tau^k \mid c)\big]\Delta \tau^k
+ g(\tau^k)\sqrt{|\Delta \tau^k|}\,\xi^k,\quad \xi^k \sim \mathcal{N}(0,I).
$}
\end{equation}
}
where $f$ and $g$ denote the drift and diffusion coefficients, and $\Delta \tau^k = \tau^{k-1} - \tau^k$ is the step size. In this work, we additionally apply a Langevin corrector \citep{song2019generative} after each predictor step to improve sample quality. When the conditional input $c$ is absent, the formulation simplifies to an unconditional score network $q_\theta(x^{\tau}, \tau)$. Implementation details are provided in Appendix~\ref{secapp:imp}.
}

\vspace{-0.1in}
\section{Methodology}
\label{sec:method}

\vspace{-0.1in}

\begin{figure*}[h]
\centering
\includegraphics[width=\textwidth]{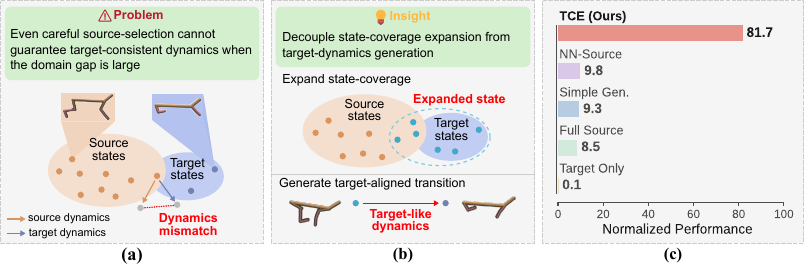}
\caption{Motivation for target-aligned coverage expansion. (a) Source and target HalfCheetah domains with a large morphology gap. (b) TCE expands limited target-state coverage using source states and generates target-aligned transitions over the expanded support. (c) With coverage expansion, TCE (Ours) substantially outperforms target-only training (Target Only), full source reuse (Full Source), nearest-neighbor-based source selection (NN-Source), and simple generation (Simple Gen.).}
\label{fig:motiv}
\end{figure*}

\vspace{-0.05in}
\subsection{Motivation}
\label{sec}
\vspace{-0.05in}

A key challenge in cross-domain offline RL is to effectively exploit abundant source data when only limited target data are available. This challenge is particularly pronounced under large dynamics gaps, as illustrated by the substantial morphology difference between the source and target HalfCheetah domains in Fig.~\ref{fig:motiv}(a). Directly incorporating source transitions, even those close to the target data, can introduce considerable dynamics gap. In contrast, simple generation from scarce target data avoids source dynamics but may remain concentrated around the observed target support, providing limited improvement in state coverage. These limitations motivate using source data not simply as additional transitions, but to expand the limited target state support for generation.

Building on this idea, TCE expands the target support toward nearby source states and generates target-aligned transitions over the expanded regions, as illustrated in Fig.~\ref{fig:motiv}(b). The extent of expansion is controlled to increase coverage without excessively increasing the resulting performance gap. In this example, the source--target dynamics gap is sufficiently large that source transitions are excluded, while their states remain valuable for expanding the generation support. As shown in Fig.~\ref{fig:motiv}(c), this strategy yields a substantial performance improvement, whereas target-only learning, direct source reuse, NN-based selection of target-near source transitions, and simple generation provide only limited gains. This illustrates how source data can remain useful even under large dynamics gaps by expanding the support of target-aligned generation beyond that provided by scarce target data.

\subsection{Theoretical Analysis}
\label{sec:TCE}
\vspace{-0.05in}

To formalize the motivation above, we analyze the performance gap induced by two sources of transition discrepancy: directly incorporating source transitions and generating target-aligned transitions beyond the limited target data. This decomposition allows us to characterize how the dynamics gap and generation error jointly affect target performance, providing a basis for expanding target coverage while controlling the resulting gap. For an MDP $\mathcal{M}$ and policy $\pi$, let $\rho_{\mathcal{M}}^\pi(s,a):=(1-\gamma)\sum_{t=0}^{\infty}\gamma^t P_{\mathcal{M}}^\pi(s_t=s,a_t=a)$ denote the normalized discounted occupancy measure, where $P_{\mathcal{M}}^\pi$ is the visitation probability under $\pi$, and define the discounted return as $J_{\mathcal{M}}(\pi):=\frac{1}{1-\gamma}\mathbb{E}_{\rho_{\mathcal{M}}^\pi}[r(s,a)]$. We assume that the source and target MDPs share the state and action space, reward function, discount factor, and initial-state distribution, with $|r(s,a)|\leq r_{\max}$.

Let $P_{\mathrm{src}}$, $P_{\mathrm{tar}}$, and $\hat P_{\mathrm{tar}}$ denote the source, target, and learned target-aligned transition kernels conditioned on $(s,a)$, respectively. Define $P_{\mathrm{mix}}:=\lambda P_{\mathrm{src}}+(1-\lambda)\hat P_{\mathrm{tar}}$, where $\lambda\in[0,1]$ is the mixture coefficient, and let $\mathcal{M}_{\mathrm{mix}}$ denote the induced MDP. Adapting the standard telescoping analysis of \citet{vgdf} to our target-aligned generative setting, we obtain the following performance-gap bound.

\begin{theorem}
\label{thm:gap}
Let $J_{\mathrm{mix}}(\pi)$ and $\rho_{\mathrm{mix}}^{\pi}$ denote the return and state-action occupancy measure under $\mathcal{M}_{\mathrm{mix}}$. Then,
{\small
\begin{equation}
J_{\mathrm{mix}}(\pi)-J_{\mathrm{tar}}(\pi)
\leq
\frac{2\gamma r_{\max}}{(1-\gamma)^2}
\Big(
\lambda\,\mathbb{E}_{\rho_{\mathrm{mix}}^{\pi}}
\big[D_{\mathrm{TV}}(P_{\mathrm{src}}\,\|\,P_{\mathrm{tar}})\big]
+
(1-\lambda)\,\mathbb{E}_{\rho_{\mathrm{mix}}^{\pi}}
\big[D_{\mathrm{TV}}(\hat P_{\mathrm{tar}}\,\|\,P_{\mathrm{tar}})\big]
\Big).
\end{equation}
}
where $D_{\mathrm{TV}}$ denotes total variation distance and the expectations are over $(s,a)\sim\rho_{\mathrm{mix}}^\pi$.

In addition, letting $V_{\mathrm{tar}}^\pi$ denote the value function of policy $\pi$ in the target MDP, the performance gap can be bounded in terms of the corresponding value discrepancies as
\begin{align}
J_{\mathrm{mix}}(\pi)-J_{\mathrm{tar}}(\pi)
\leq \frac{\gamma}{1-\gamma}\Big(
&\lambda\,\mathbb{E}_{\rho_{\mathrm{mix}}^\pi}\!\left[
\left|
\mathbb{E}_{P_{\mathrm{src}}}[V_{\mathrm{tar}}^\pi(s')]
-\mathbb{E}_{P_{\mathrm{tar}}}[V_{\mathrm{tar}}^\pi(s')]
\right|
\right] \nonumber\\
&+(1-\lambda)\,\mathbb{E}_{\rho_{\mathrm{mix}}^\pi}\!\left[
\left|
\mathbb{E}_{\hat P_{\mathrm{tar}}}[V_{\mathrm{tar}}^\pi(s')]
-\mathbb{E}_{P_{\mathrm{tar}}}[V_{\mathrm{tar}}^\pi(s')]
\right|
\right]
\Big).
\label{eq:value_gap_bound}
\end{align}
\end{theorem}

\vspace{-0.15in}
\noindent\textbf{Proof.}
The proof of Theorem~\ref{thm:gap} is provided in Appendix~\ref{secapp:theorem-proof}.

Theorem~\ref{thm:gap} decomposes the performance gap over the expanded occupancy $\rho_{\mathrm{mix}}^\pi$ into two terms: \textit{the source--target dynamics gap} induced by $P_{\mathrm{src}}$ and \textit{the generation error} induced by $\hat P_{\mathrm{tar}}$, weighted by $\lambda$ and $1-\lambda$, respectively. Thus, effective coverage expansion requires not only broadening the target support, but also keeping the resulting performance gap small. This leads to two key implications for TCE. First, target-aligned generation should expand coverage as much as possible without inducing excessive generation error. Second, the source--target dynamics gap determines whether selected source transitions should be additionally incorporated. When this gap is small, target-near source transitions can complement the generated data; when it is large, their direct use becomes less favorable, and source data are instead used primarily for coverage expansion. These principles motivate the design of TCE presented next and are further examined across different dynamics gaps in our experiments.

\vspace{-0.05in}
\subsection{Target-Aligned Generation over Expanded State Support}
\vspace{-0.05in}
\label{subsec:sdes_TCE}
We first address the primary objective suggested by Theorem~\ref{thm:gap}: expanding the limited target-state coverage without excessively increasing the generation-induced performance gap. TCE separates state-support expansion from target-transition generation using two score networks: a \textit{coverage-expansion score network} $q^{\mathrm{cov}}_\theta$ for expanding the state distribution toward source-supported regions, and a \textit{target-transition score network} $q^{\mathrm{tran}}_\theta$ for generating target-aligned state pairs $(s,s')$ over the expanded support. The action and reward are subsequently inferred from each generated state pair rather than jointly generating the full transition $(s,a,r,s')$ conditioned on $s$. As shown empirically in Section~\ref{sec:exp}, this state-pair-first design provides better generalization under scarce target data.

To control the extent of coverage expansion, we introduce a coverage coefficient $\eta\in[0,1]$ and progressively incorporate source states according to their proximity to the target data. For a source transition $(s,s')\in\mathcal D_{\mathrm{src}}$, we define its distance to the target dataset as
\begin{equation}
d_{\mathrm{NN}}(s,s')
:=
\min_{(\tilde s,\tilde s')\in\mathcal D_{\mathrm{tar}}}
\|s-\tilde s\|+\|s'-\tilde s'\|.
\end{equation}
For $\eta\in(0,1]$, let $d_{\eta}$ denote the $\eta$-quantile of these distances and define
\begin{equation}
\label{equ:NN_src}
\mathcal D_{\mathrm{src}}^{\eta}
:=
\left\{
(s,s')\in\mathcal D_{\mathrm{src}}:
d_{\mathrm{NN}}(s,s')\leq d_{\eta}
\right\}.
\end{equation}
When $\eta=0$, only target states are used, whereas increasing $\eta$ progressively expands the support toward source-covered regions, reaching the full source state support at $\eta=1$. Accordingly, $q^{\mathrm{cov}}_\theta$ is trained on states from $\mathcal D_{\mathrm{src}}^\eta\cup\mathcal D_{\mathrm{tar}}$ to model the expanded state distribution, while $q^{\mathrm{tran}}_\theta$ is trained exclusively on $\mathcal D_{\mathrm{tar}}$ to generate target-aligned state pairs over this expanded support.

Building on the score-matching objective in Eq.~(\ref{eq:baseloss}), we train the two score networks as
{\small
\begin{equation}
\label{eq:joint_loss}
\resizebox{0.94\columnwidth}{!}{$
\underbrace{
\mathbb{E}_{\tau,\,s_t\sim\mathcal D_{\mathrm{src}}^{\eta}\cup\mathcal D_{\mathrm{tar}}}
\!\left[
\zeta(\tau)
\left\|
q^{\mathrm{cov}}_\theta(s_t^\tau,\tau)
+\tfrac{z}{\sigma(\tau)}
\right\|_2^2
\right]
}_{\textbf{\small coverage-expansion score loss}}
+
\underbrace{
\mathbb{E}_{\tau,\,(s_t,s_{t+1})\sim\mathcal D_{\mathrm{tar}}}
\!\left[
\zeta(\tau)
\left\|
q^{\mathrm{tran}}_\theta(s_{t+1}^\tau,\tau\mid s_t)
+\tfrac{z}{\sigma(\tau)}
\right\|_2^2
\right]
}_{\textbf{\small target-transition score loss}}
,$}
\raisetag{0.8em}
\end{equation}
}
where $z\sim\mathcal N(0,I)$ and $s^\tau=s+\sigma(\tau)z$.

After training, we first sample a state from the expanded support by solving the reverse SDE with the coverage-expansion score network:
\begin{equation}
s^{k-1}=s^k+\Big[f(s^k,\tau^k)-g(\tau^k)^2q^{\mathrm{cov}}_\theta(s^k,\tau^k)\Big]\Delta\tau^k
+g(\tau^k)\sqrt{|\Delta\tau^k|}\xi^k,\quad
\xi^k\sim\mathcal N(0,I),
\end{equation}
which is integrated from $k=K$ to $0$ to obtain $\hat s_t=s^0$. Conditioned on $\hat s_t$, we then solve the reverse SDE with the target-transition score network:
\begin{equation}
s^{k-1}=s^k+\Big[f(s^k,\tau^k)-g(\tau^k)^2q^{\mathrm{tran}}_\theta(s^k,\tau^k\mid\hat s_t)\Big]\Delta\tau^k
+g(\tau^k)\sqrt{|\Delta\tau^k|}\xi^k,\quad
\xi^k\sim\mathcal N(0,I),
\end{equation}
yielding a target-aligned state pair $(\hat s_t,\hat s_{t+1})$, where $\hat s_{t+1}=s^0$. We then generate the corresponding action and reward using separately trained inverse dynamics and reward models, $\mathrm{Inv}_\psi$ and $\hat R_\phi$, as $\hat a_t\sim\mathrm{Inv}_\psi(\cdot\mid\hat s_t,\hat s_{t+1})$ and $\hat r_t=\hat R_\phi(\hat s_t,\hat a_t,\hat s_{t+1})$, respectively.

This procedure yields the \textit{coverage-expanded, target-aligned dataset}
$\mathcal D_{\mathrm{gen}}^{\eta}
=\{(\hat s_t,\hat a_t,\hat r_t,\hat s_{t+1})\}$.
The coefficient $\eta$ directly controls its coverage: increasing $\eta$ expands $\mathcal D_{\mathrm{gen}}^{\eta}$ toward broader source-supported regions, but may also increase the error of $\hat P_{\mathrm{tar}}$ farther from the observed target support. Thus, $\eta$ governs the trade-off between coverage expansion and the generation-induced performance gap characterized in Theorem~\ref{thm:gap}.

\vspace{-0.05in}
\subsection{TCE Dataset Construction and Offline Learning}
\vspace{-0.05in}
\label{subsec:training}
While $\mathcal D_{\mathrm{gen}}^\eta$ expands target coverage without directly inheriting source dynamics, target-near source transitions may also provide useful data for offline learning. Following the NN criterion in Eq.~(\ref{equ:NN_src}), we define $\mathcal D_{\mathrm{src}}^\lambda$ as the closest $\lambda$ fraction of source transitions. Under a fixed augmentation budget, we construct the TCE training dataset as
\[
\mathcal{D}_{\mathrm{train}}^{\eta,\lambda}
=
\mathcal{D}_{\mathrm{tar}}
\cup
\mathcal{D}_{\mathrm{src}}^\lambda
\cup
\mathcal{D}_{\mathrm{gen}}^\eta,
\]
where
$|\mathcal D_{\mathrm{src}}^\lambda|=\lambda|\mathcal D_{\mathrm{src}}|$ and
$|\mathcal D_{\mathrm{gen}}^\eta|=(1-\lambda)|\mathcal D_{\mathrm{src}}|$.
Here, $\lambda\in[0,1]$ corresponds to the mixture coefficient in Theorem~\ref{thm:gap} and controls the allocation between direct source utilization and target-aligned generation. A larger $\lambda$ increases the contribution of source transitions and thus the influence of the source--target dynamics gap, whereas a smaller $\lambda$ places greater emphasis on generated transitions and their generation error. 
The training dataset can therefore be adjusted according to the relative \begin{wrapfigure}{r}{0.48\textwidth}
    \vspace{-0.1in}
    \centering
    \includegraphics[width=\linewidth]{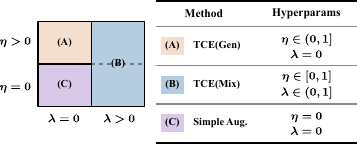}
    \vspace{-0.25in}
    \caption{TCE configurations controlled by the coverage coefficient $\eta$ and mixture coefficient $\lambda$.}
    \label{fig:eta_lambda_roles}
    \vspace{-0.25in}
\end{wrapfigure} magnitudes of the dynamics gap and generation error.

Accordingly, we consider two variants of TCE, as illustrated in Fig.~\ref{fig:eta_lambda_roles}. \textbf{TCE(Gen)} uses only coverage-expanded target-aligned generation ($\lambda=0$, $\eta>0$), whereas \textbf{TCE(Mix)} additionally incorporates target-near source transitions ($\lambda>0$, $\eta\geq0$). TCE(Gen) is used under large source--target dynamics gaps, while TCE(Mix) allows direct source utilization when the dynamics gap of selected source transitions is sufficiently small. The setting $\eta=\lambda=0$ corresponds to target-only generation without coverage expansion, which we denote as \textbf{Simple Augmentation (Simple Aug.)}.

Using $\mathcal{D}_{\mathrm{train}}^{\eta,\lambda}$, we perform offline policy learning with Implicit Q-Learning (IQL) \citep{iql}. Following \citet{lyu2025cross}, we additionally regularize the policy toward the target-domain behavior policy:
\begin{equation}
\mathcal{L}_\pi
=
\mathcal{L}_\pi^{\mathrm{IQL}}
+
\beta~
\mathbb{E}_{s \sim \mathcal{D}_{\mathrm{tar}}}
\big[
D_{\mathrm{KL}}
(\hat{\pi}_{\mathrm{b}}(\cdot|s)\,\|\,\pi(\cdot|s))
\big],
\end{equation}
where $\hat{\pi}_{\mathrm{b}}$ denotes the empirical behavior policy estimated from the target dataset. The overall TCE framework is illustrated in Fig.~\ref{fig:structure}, with the complete procedure summarized in Algorithm~\ref{alg}. Further implementation details are provided in Appendix~\ref{secapp:imp}.

\begin{figure*}[t]
    \centering
    \vspace{-0.1in}
    \begin{minipage}[t]{0.47\textwidth}
        \vspace{0pt}
        \centering
        {
        \includegraphics[width=\linewidth]{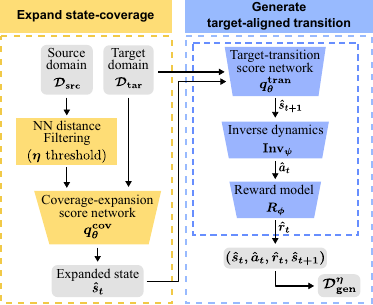}
        \vspace{-0.18in}
        \caption{Overview of the TCE generation.}
        \label{fig:structure}
        }
    \end{minipage}
    \hfill
    \begin{minipage}[t]{0.49\textwidth}
        \vspace{-1.2em}
        \begin{algorithm}[H]
            \caption{TCE Framework}
            \label{alg}
            \small
            \begin{algorithmic}[1]
            \STATE {\textbf{Input:} $\mathcal{D}_{\mathrm{tar}}, \mathcal{D}_{\mathrm{src}}, \eta, \lambda$}
            \STATE \textbf{Train models:} Train $q^{\mathrm{cov}}_\theta$ on $\mathcal{D}_{\mathrm{src}}^{\eta}\cup\mathcal{D}_{\mathrm{tar}}$,
            \STATE and $q^{\mathrm{tran}}_\theta$ on $\mathcal{D}_{\mathrm{tar}}$ via Eq.~(\ref{eq:joint_loss})
            \STATE Train $\mathrm{Inv}_\psi$ on $\mathcal{D}_{\mathrm{tar}}$, and $\hat{R}_\phi$ on $\mathcal{D}_{\mathrm{src}}\cup\mathcal{D}_{\mathrm{tar}}$
            \STATE \textbf{Generate samples:}
            \STATE Generate $\hat{s}_t\sim q^{\mathrm{cov}}_\theta$, $\hat{s}_{t+1}\sim q^{\mathrm{tran}}_\theta(\cdot\mid\hat{s}_t)$,
            \STATE {$\hat{a}_t\sim \mathrm{Inv}_\psi$, and $\hat{r}_t=\hat{R}_\phi(\hat{s}_t,\hat{a}_t,\hat{s}_{t+1})$ to form $\mathcal D_{\mathrm{gen}}^\eta$, where $|\mathcal D_{\mathrm{gen}}^\eta|=(1-\lambda)|\mathcal D_{\mathrm{src}}|$}
            \STATE \textbf{Construct training data:}
            \STATE {Let $\mathcal D_{\mathrm{src}}^\lambda$ be the $\lambda|\mathcal D_{\mathrm{src}}|$ source transitions with the smallest $d_{\mathrm{NN}}$}
            \STATE {$\mathcal{D}_{\mathrm{train}}^{\eta,\lambda}=\mathcal{D}_{\mathrm{tar}}\cup\mathcal{D}_{\mathrm{src}}^\lambda\cup\mathcal{D}_{\mathrm{gen}}^\eta$}
            \STATE \textbf{Offline RL:} Train $\pi$ on $\mathcal{D}_{\mathrm{train}}^{\eta,\lambda}$
            \end{algorithmic}
        \end{algorithm}
    \end{minipage}
\vspace{-0.15in}
\end{figure*}

\vspace{-0.1in}
\section{Experiments}
\label{sec:exp}
\vspace{-0.1in}

In this section, we evaluate TCE across diverse cross-domain setups and compare it with recent cross-domain offline RL baselines. We further perform ablations to quantify the contribution of each component and study sensitivity to key hyperparameters.

\vspace{-0.1in}
\subsection{Experimental Setup}
\vspace{-0.05in}
{
We evaluate TCE on cross-domain MuJoCo continuous-control tasks \citep{todorov2012mujoco} and Adroit manipulation tasks from the ODRL benchmark \citep{lyu2024odrl}.
In MuJoCo, the source and target domains share the same agent type (HalfCheetah, Hopper, Walker2d, or Ant) but differ through morphology, kinematic, or gravity shifts. Morphology shifts alter body-segment dimensions, kinematic shifts restrict joint rotation ranges, and gravity shifts halve the gravitational acceleration while preserving the agent structure.
We use pre-collected D4RL datasets \citep{d4rl} of varying quality: \texttt{medium}, \texttt{medium-expert}, and \texttt{expert} for the target domain, and \texttt{medium}, \texttt{medium-replay}, and \texttt{medium-expert} for the source domain. This results in $36$ cross-domain tasks for each shift setting, with $0.2$M--$2$M source transitions and at most 5K target transitions. In Adroit, we consider morphological (\texttt{shrink-finger}) and kinematic (\texttt{broken-joint}) shifts at medium and hard target constraint levels (M/H). We use expert-only datasets containing $0.5$M source transitions and 5K target transitions. We report normalized returns, where $0$ and $100$ correspond to random and expert policies, respectively.
}

{ For comparison, we evaluate TCE against cross-domain offline RL baselines: \textbf{IQL*} \citep{iql}, which trains IQL on source-target data;
\textbf{DARA} \citep{dara}, which
{applies classifier-based reward correction} to mitigate dynamics mismatch;
\textbf{BOSA} \citep{bosa}, which constrains policies to dataset support;
\textbf{IGDF} \citep{wen2024contrastive}, which filters source transitions via representation learning;
\textbf{MOBODY}~\citep{mobody}, a generation-based approach that learns target dynamics to augment data for policy optimization;
and \textbf{OTDF} \citep{lyu2025cross}, which filters source data using optimal transport distances. For Adroit, we use the official ODRL benchmark~\citep{lyu2024odrl}.}

{ For IQL*, DARA, BOSA, IGDF, and OTDF, we report results from \citep{lyu2025cross}, where each algorithm is implemented with hyperparameter tuning. MOBODY is evaluated using its official implementation. For TCE(Gen) and TCE(Mix), we use the coverage coefficient $\eta=0.2$ unless otherwise specified. TCE(Gen) fixes $\lambda=0$, whereas TCE(Mix) selects $\lambda$ from $\{0.1,0.5,0.9\}$. More details, including hyperparameter settings, are provided in Appendix~\ref{secapp:expdetail}.}

{
\begin{table*}[!t]
\centering
\vspace{-0.1in}
\caption{Performance comparison on morphology-shift tasks. Abbrev.: half=HalfCheetah, hopp=Hopper, walk=Walker2d, ant=Ant; m=\texttt{medium}, m-r=\texttt{medium-replay}, e=\texttt{expert}, m-e=\texttt{medium-expert}. Src./Tgt. denote source/target domains. Results are reported as mean $\pm$ standard deviation over 5 seeds; best results are in \textbf{bold} and second-best results are \underline{underlined}.}
\label{table:perform_morph}
\vspace{-0.1in}
\small
\setlength{\tabcolsep}{4pt}
\renewcommand{\arraystretch}{1.08}
\resizebox{\textwidth}{!}{%
\begin{tabular}{@{}S{1.25cm}l|cccccc|cc@{}}
\toprule
Src. & Tgt. & IQL* & DARA & BOSA & IGDF & MOBODY & OTDF & TCE(Gen) & TCE(Mix) \\
\midrule
half-m & m & 30.0$\pm$1.6 & 26.6$\pm$3.3 & 19.3$\pm$3.5 & 41.6$\pm$0.5 & \underline{42.5$\pm$1.2} & 39.1$\pm$2.3 & \textbf{44.1$\pm$0.2} & \textbf{44.1$\pm$0.2} \\
half-m & m-e & 31.8$\pm$1.1 & 32.0$\pm$0.7 & 33.6$\pm$1.1 & 29.6$\pm$2.2 & 26.4$\pm$2.5 & 35.6$\pm$0.7 & \textbf{44.7$\pm$0.6} & \underline{43.3$\pm$0.5} \\
half-m & e & 8.5$\pm$1.0 & 9.3$\pm$1.6 & 7.9$\pm$0.8 & 10.0$\pm$0.8 & 9.3$\pm$0.8 & 10.7$\pm$1.2 & \textbf{81.7$\pm$9.2} & \underline{61.5$\pm$4.5} \\
half-m-r & m & 30.8$\pm$4.4 & 35.6$\pm$0.7 & 35.0$\pm$4.6 & 28.0$\pm$2.0 & 41.1$\pm$0.6 & 40.0$\pm$1.2 & \textbf{44.2$\pm$1.0} & \underline{43.3$\pm$2.3} \\
half-m-r & m-e & 12.9$\pm$2.2 & 16.9$\pm$4.1 & 19.9$\pm$5.5 & 12.0$\pm$3.7 & 22.3$\pm$2.1 & 34.4$\pm$0.7 & \textbf{44.2$\pm$1.3} & \underline{43.1$\pm$0.1} \\
half-m-r & e & 5.9$\pm$1.7 & 3.7$\pm$2.7 & 2.4$\pm$1.9 & 5.3$\pm$2.3 & 4.3$\pm$2.4 & 8.2$\pm$2.7 & \textbf{80.2$\pm$0.8} & \underline{36.9$\pm$8.6} \\
half-m-e & m & 41.5$\pm$0.1 & 40.3$\pm$1.2 & 41.3$\pm$0.3 & 40.9$\pm$0.4 & \underline{42.3$\pm$1.7} & 41.4$\pm$0.3 & \textbf{44.1$\pm$1.3} & \textbf{44.1$\pm$2.0} \\
half-m-e & m-e & 25.8$\pm$2.0 & 30.6$\pm$2.8 & 32.1$\pm$0.8 & 26.2$\pm$1.8 & 27.5$\pm$2.8 & 35.1$\pm$0.6 & \textbf{44.8$\pm$0.5} & \underline{43.8$\pm$2.0} \\
half-m-e & e & 7.8$\pm$1.3 & 8.3$\pm$1.3 & 9.1$\pm$0.8 & 7.5$\pm$0.9 & 9.9$\pm$1.8 & 9.8$\pm$1.0 & \textbf{83.2$\pm$5.1} & \underline{62.9$\pm$8.3} \\ \hline
hopp-m & m & \underline{13.5$\pm$0.2} & \underline{13.5$\pm$0.4} & 13.2$\pm$0.3 & 13.4$\pm$0.2 & 13.3$\pm$0.5 & 11.0$\pm$0.9 & \textbf{50.0$\pm$2.0} & 5.0$\pm$3.7 \\
hopp-m & m-e & 13.4$\pm$0.1 & \underline{13.6$\pm$0.2} & 11.2$\pm$4.6 & 13.3$\pm$0.4 & 13.0$\pm$0.9 & 12.6$\pm$0.8 & \textbf{16.1$\pm$5.3} & 9.2$\pm$1.8 \\
hopp-m & e & 13.5$\pm$0.2 & 13.6$\pm$0.3 & 13.3$\pm$0.4 & 13.9$\pm$0.1 & 13.4$\pm$0.4 & 10.7$\pm$4.7 & \textbf{99.8$\pm$0.1} & \underline{77.3$\pm$15.7} \\
hopp-m-r & m & 10.8$\pm$1.1 & 10.2$\pm$1.0 & 1.2$\pm$0.0 & 12.0$\pm$4.4 & 5.6$\pm$3.4 & 8.7$\pm$2.8 & \textbf{31.1$\pm$3.2} & \underline{27.0$\pm$14.5} \\
hopp-m-r & m-e & \underline{11.6$\pm$1.6} & 10.4$\pm$0.9 & 1.3$\pm$0.2 & 8.2$\pm$2.8 & \underline{11.6$\pm$2.3} & 9.7$\pm$2.7 & \textbf{15.8$\pm$4.1} & 3.4$\pm$0.1 \\
hopp-m-r & e & 9.8$\pm$0.5 & 9.0$\pm$0.3 & 1.3$\pm$0.1 & 11.4$\pm$1.5 & 9.4$\pm$1.8 & 10.7$\pm$2.4 & \textbf{99.8$\pm$0.1} & \underline{75.4$\pm$26.1} \\
hopp-m-e & m & 12.6$\pm$1.4 & 13.0$\pm$0.5 & 15.7$\pm$7.2 & 12.7$\pm$0.8 & 13.6$\pm$0.9 & 7.9$\pm$3.2 & \textbf{50.5$\pm$7.6} & \underline{22.0$\pm$17.3} \\
hopp-m-e & m-e & \underline{14.1$\pm$1.3} & 13.8$\pm$0.6 & 12.0$\pm$1.4 & 13.3$\pm$1.2 & 13.2$\pm$2.7 & 9.6$\pm$3.5 & \textbf{18.8$\pm$6.5} & 10.2$\pm$1.3 \\
hopp-m-e & e & 13.8$\pm$0.5 & 12.3$\pm$1.8 & 10.5$\pm$5.0 & 12.8$\pm$0.9 & 13.1$\pm$2.2 & 5.9$\pm$4.0 & \textbf{97.9$\pm$0.3} & \underline{66.1$\pm$23.1} \\ \hline
walk-m & m & 23.0$\pm$4.7 & 23.3$\pm$3.3 & 6.2$\pm$2.9 & 27.5$\pm$9.5 & 36.6$\pm$9.7 & \underline{50.5$\pm$5.8} & 43.2$\pm$3.8 & \textbf{55.2$\pm$5.3} \\
walk-m & m-e & 21.5$\pm$8.6 & 22.2$\pm$7.6 & 7.2$\pm$2.9 & 20.7$\pm$5.9 & 20.6$\pm$9.4 & \textbf{44.3$\pm$23.8} & \underline{40.4$\pm$2.5} & 22.7$\pm$9.7 \\
walk-m & e & 20.3$\pm$2.8 & 17.3$\pm$3.4 & 15.8$\pm$8.7 & 15.8$\pm$4.5 & 31.5$\pm$15.6 & 55.3$\pm$8.3 & \underline{74.2$\pm$1.6} & \textbf{86.8$\pm$11.8} \\
walk-m-r & m & 11.3$\pm$3.0 & 10.9$\pm$4.6 & 5.4$\pm$4.0 & 13.4$\pm$7.2 & 27.0$\pm$13.4 & 37.4$\pm$5.1 & \underline{44.7$\pm$2.7} & \textbf{45.6$\pm$10.4} \\
walk-m-r & m-e & 7.0$\pm$1.5 & 4.5$\pm$1.1 & 4.0$\pm$2.2 & 6.9$\pm$2.2 & 17.5$\pm$12.8 & \underline{33.8$\pm$6.9} & \textbf{43.4$\pm$2.6} & 17.2$\pm$3.1 \\
walk-m-r & e & 6.3$\pm$0.9 & 4.5$\pm$1.1 & 3.8$\pm$3.4 & 5.5$\pm$2.2 & 24.5$\pm$15.1 & 41.5$\pm$6.8 & \textbf{63.2$\pm$7.4} & \underline{52.6$\pm$0.2} \\
walk-m-e & m & 24.1$\pm$7.4 & 31.7$\pm$6.6 & 18.7$\pm$6.5 & 27.5$\pm$2.3 & 37.7$\pm$18.5 & \textbf{49.9$\pm$4.6} & 43.5$\pm$3.4 & \underline{47.4$\pm$3.6} \\
walk-m-e & m-e & 27.0$\pm$5.5 & 23.3$\pm$5.5 & 11.1$\pm$0.9 & 25.3$\pm$6.4 & 24.1$\pm$12.1 & \textbf{40.5$\pm$11.0} & \underline{37.4$\pm$6.1} & 21.7$\pm$7.9 \\
walk-m-e & e & 22.4$\pm$3.3 & 25.2$\pm$5.7 & 9.9$\pm$3.9 & 24.7$\pm$2.4 & 26.2$\pm$9.6 & 45.7$\pm$6.9 & \textbf{62.5$\pm$9.6} & \underline{47.6$\pm$10.9} \\ \hline
ant-m & m & 38.7$\pm$3.8 & \underline{41.3$\pm$1.8} & 18.2$\pm$1.9 & 40.9$\pm$1.7 & 40.6$\pm$0.8 & 39.4$\pm$1.7 & \textbf{41.5$\pm$0.8} & 40.8$\pm$1.2 \\
ant-m & m-e & 47.0$\pm$5.1 & 43.3$\pm$2.0 & 45.3$\pm$7.0 & 44.4$\pm$1.7 & 54.7$\pm$2.5 & 58.3$\pm$8.9 & \underline{72.7$\pm$2.5} & \textbf{74.9$\pm$4.0} \\
ant-m & e & 36.2$\pm$3.5 & 48.5$\pm$4.2 & 72.2$\pm$10.5 & 41.4$\pm$4.2 & 58.0$\pm$7.4 & 85.4$\pm$4.4 & \underline{94.0$\pm$4.5} & \textbf{96.4$\pm$3.2} \\
ant-m-r & m & 38.2$\pm$2.9 & 38.9$\pm$2.7 & 20.2$\pm$3.7 & 39.7$\pm$1.2 & 40.8$\pm$2.8 & \underline{41.2$\pm$0.9} & \textbf{41.5$\pm$1.5} & 40.0$\pm$1.6 \\
ant-m-r & m-e & 38.1$\pm$3.5 & 33.4$\pm$5.5 & 15.2$\pm$1.6 & 37.3$\pm$2.4 & 49.4$\pm$3.6 & 50.8$\pm$4.5 & \textbf{71.0$\pm$8.4} & \underline{63.1$\pm$9.8} \\
ant-m-r & e & 24.1$\pm$1.9 & 24.5$\pm$2.6 & 16.0$\pm$1.7 & 23.6$\pm$1.4 & 32.4$\pm$8.7 & \underline{67.2$\pm$7.5} & \textbf{96.1$\pm$0.4} & \textbf{96.1$\pm$0.6} \\
ant-m-e & m & 32.9$\pm$5.1 & 40.2$\pm$1.5 & 28.1$\pm$5.6 & 36.1$\pm$4.4 & 39.5$\pm$3.5 & 39.9$\pm$2.9 & \textbf{41.9$\pm$0.6} & \underline{41.4$\pm$1.8} \\
ant-m-e & m-e & 35.7$\pm$3.9 & 36.5$\pm$8.7 & 14.8$\pm$15.9 & 30.7$\pm$10.8 & 50.3$\pm$9.2 & 65.7$\pm$4.5 & \underline{66.4$\pm$1.2} & \textbf{69.2$\pm$4.9} \\
ant-m-e & e & 36.1$\pm$8.5 & 34.6$\pm$5.8 & 53.9$\pm$5.0 & 35.2$\pm$6.6 & 44.9$\pm$6.5 & 86.4$\pm$2.2 & \textbf{96.6$\pm$0.5} & \underline{94.5$\pm$2.4} \\
\midrule
\multicolumn{2}{c|}{\textbf{Total Score}} & 798.0 & 816.8 & 646.3 & 808.7 & 988.1 & 1274.3 & \textbf{2065.2} & \underline{1731.8} \\
\bottomrule
\end{tabular}
}
\vspace{-0.1in}
\end{table*}
}

\vspace{-0.05in}
\subsection{Performance Comparison}

\vspace{-0.05in}

\paragraph{Morphology Shifts.}
Table~\ref{table:perform_morph} shows that TCE consistently outperforms the baselines across morphology shifts, with at least one TCE variant achieving the highest mean score on $33$ of $36$ tasks. TCE also achieves a total score of $2065.2$, substantially higher than the strongest baseline, OTDF, at $1274.3$. In particular, TCE shows pronounced gains on \texttt{expert}-target tasks, where existing methods often struggle under limited target coverage. Morphology changes induce relatively large dynamics gaps by directly altering the physical structure of the agent. In these settings, the dynamics gap is generally large relative to the generation error, and TCE(Gen) outperforms TCE(Mix) on most tasks. We further analyze this relationship in Section~\ref{subsec:analysis}.

\vspace{-0.1in}
\paragraph{Kinematic and Gravity Shifts.}
Full results for kinematic and gravity shifts are provided in Appendix~\ref{secapp:perf_comp}. Kinematic shifts exhibit a similar trend to morphology shifts due to their relatively large dynamics gaps: TCE achieves the highest mean score on $30$ of $36$ tasks, with TCE(Gen) generally outperforming TCE(Mix). In contrast, gravity shifts induce smaller dynamics gaps while preserving agent morphology and joint ranges. TCE remains competitive across these tasks, but TCE(Mix) generally outperforms TCE(Gen), indicating greater benefits from directly incorporating target-near source transitions.

To further examine the effect of the dynamics gap, we progressively vary the target gravity while fixing the source gravity in Appendix~\ref{subsecapp:gravity_sweep}. TCE(Mix) performs better when the gravity difference is small, whereas TCE(Gen) becomes increasingly favorable as the difference grows. This transition supports our finding that \textit{direct source reuse is more beneficial under small dynamics gaps, whereas large dynamics gaps favor target-aligned generation.}

\vspace{-0.1in}
\paragraph{Adroit Tasks.}
Across the Adroit settings in Table~\ref{tab:adroit}, TCE(Gen) achieves the highest mean score in every setting, outperforming all baselines and TCE(Mix). Its advantage over TCE(Mix) is particularly pronounced on Pen, while the two variants remain comparable in some Door settings. In these high-dimensional manipulation tasks, changes in finger morphology and joint constraints can induce larger source--target dynamics gaps, making direct source reuse less favorable and providing greater benefits to TCE(Gen). These results demonstrate the effectiveness of target-aligned generation in challenging high-dimensional manipulation settings.

{
\begin{table*}[!t]
\centering
\vspace{-0.1in}
\caption{Performance comparison on Adroit tasks from the ODRL benchmark.}
\vspace{-0.1in}
\label{tab:adroit}

\small
\resizebox{\textwidth}{!}{%
\begin{tabular}{@{}l l l|cccccc|cc@{}}
\toprule
Env & Shift & Level & IQL* & DARA & BOSA & IGDF & MOBODY & OTDF & TCE(Gen) & TCE(Mix) \\
\midrule
\multirow{4}{*}{Pen}
& \multirow{2}{*}{broken-joint}
& M & 15.5$\pm$3.4 & 33.8$\pm$2.8 & 20.4$\pm$2.5 & 29.8$\pm$2.8 & \underline{41.3$\pm$3.5} & 36.6$\pm$2.9 & \textbf{59.6$\pm$3.7} & {30.5$\pm$0.4} \\
& & H & 24.1$\pm$1.6 & 13.3$\pm$3.3 & 19.8$\pm$1.3 & 14.0$\pm$1.9 & 34.5$\pm$3.1 & \underline{40.1$\pm$4.1} & \textbf{44.8$\pm$2.2} & {23.6$\pm$1.6} \\
& \multirow{2}{*}{shrink-finger}
& M & 18.6$\pm$2.2 & 16.2$\pm$3.3 & 17.4$\pm$1.3 & 22.2$\pm$0.9 & 24.9$\pm$2.0 & \underline{34.5$\pm$5.0} & \textbf{38.6$\pm$3.1} & {25.0$\pm$12.0} \\
& & H & 16.4$\pm$1.7 & 24.5$\pm$2.8 & 9.5$\pm$1.6 & 23.5$\pm$7.2 & 29.7$\pm$1.5 & \underline{32.4$\pm$3.9} & \textbf{51.9$\pm$4.2} & {20.4$\pm$3.7} \\
\midrule
\multirow{4}{*}{Door}
& \multirow{2}{*}{broken-joint}
& M & 50.8$\pm$2.4 & 50.9$\pm$6.2 & 47.8$\pm$5.4 & 73.2$\pm$8.8 & 79.4$\pm$1.8 & 80.4$\pm$1.6 & \textbf{82.8$\pm$1.3} & {\underline{82.6$\pm$1.1}} \\
& & H & 50.7$\pm$7.2 & 36.2$\pm$4.5 & 8.8$\pm$2.1 & 53.4$\pm$7.6 & 54.3$\pm$2.3 & 78.5$\pm$2.1 & \textbf{91.8$\pm$0.9} & {\underline{81.0$\pm$0.2}} \\
& \multirow{2}{*}{shrink-finger}
& M & 28.3$\pm$5.2 & 58.6$\pm$8.4 & 0.8$\pm$0.2 & \underline{73.2$\pm$9.1} & 50.6$\pm$0.9 & 70.2$\pm$1.1 & \textbf{77.2$\pm$3.0} & {71.4$\pm$5.3} \\
& & H & 14.1$\pm$7.7 & 8.8$\pm$1.9 & 9.3$\pm$1.0 & 35.0$\pm$6.5 & 42.6$\pm$2.4 & 56.5$\pm$2.0 & \textbf{66.9$\pm$0.8} & {\underline{57.5$\pm$4.8}} \\
\midrule
\multicolumn{3}{c|}{\textbf{Total Score}} & 218.5 & 242.3 & 133.8 & 324.3 & 357.3 & \underline{429.2} & {\textbf{513.6}} & {392.0} \\
\bottomrule
\end{tabular}
}

\end{table*}
}

\subsection{Analyzing the Source--Generation Trade-off}
\vspace{-0.1in}
\label{subsec:analysis}

\begin{figure*}[t]
    \centering
    \begin{subfigure}[t]{0.365\textwidth}
        \vspace{0pt}
        \centering
        \includegraphics[width=\linewidth]{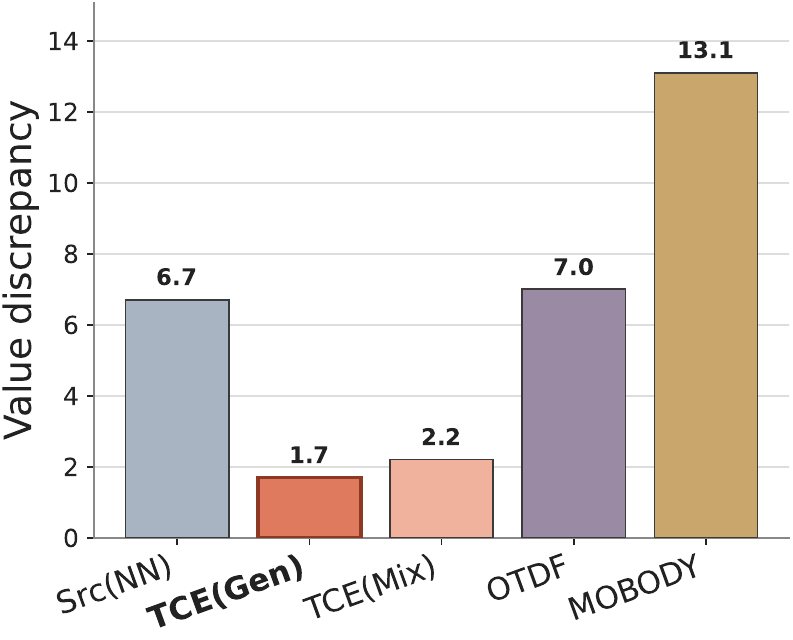}
        \caption{Value discrepancy under morphology shifts.}
        \label{fig:gapbound}
    \end{subfigure}
    \hfill
    \begin{subfigure}[t]{0.59\textwidth}
        \vspace{0pt}
        \centering
        \includegraphics[width=\linewidth]{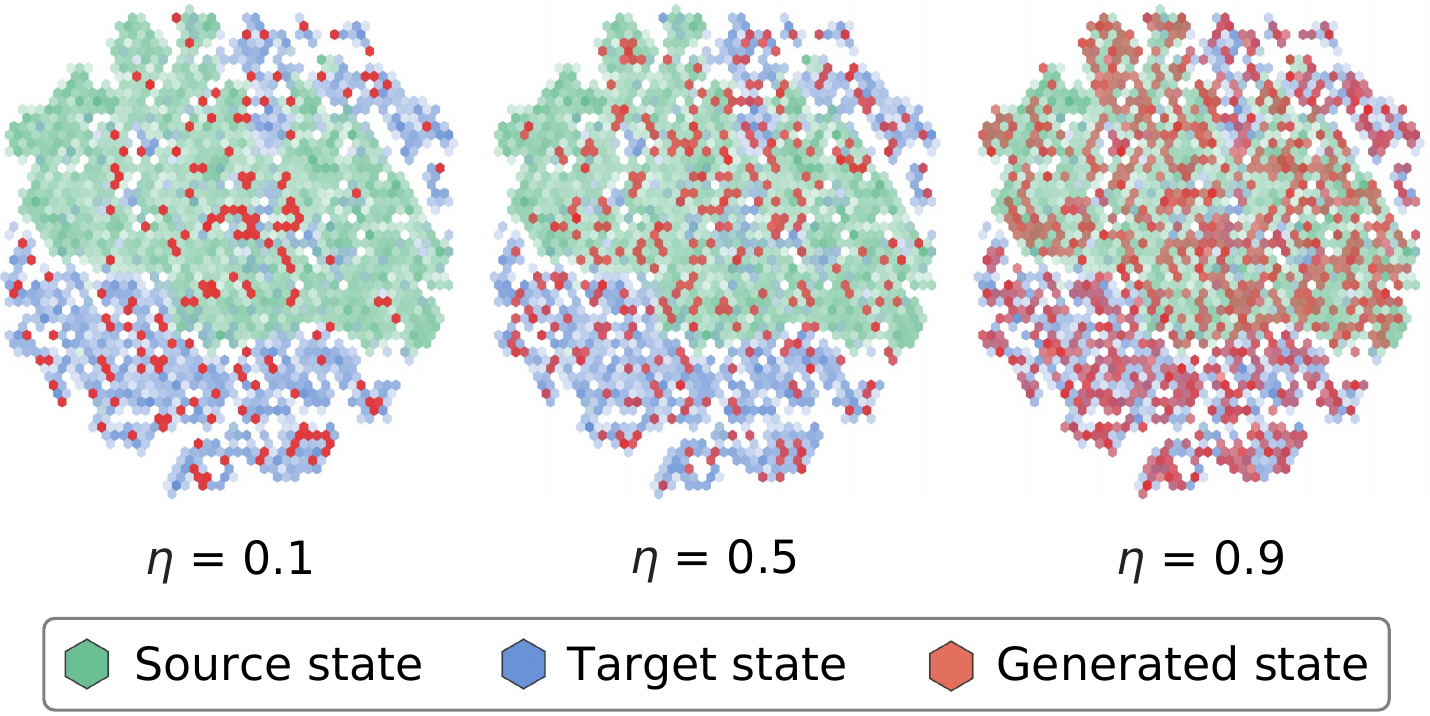}
        \caption{Coverage expansion with varying $\eta$ under morphology shifts.}
        \label{fig:analysis_sample}
    \end{subfigure}
    \vspace{-0.1in}
    \caption{Analysis of value discrepancy and coverage expansion.}
    \label{fig:coverage_reliability}
    \vspace{-0.15in}
\end{figure*}

To better understand why TCE achieves strong performance, we investigate two questions motivated by our theoretical analysis: (1) Can the performance trends of TCE be explained by the dynamics gap and generation error identified in Theorem~\ref{thm:gap}? (2) How does generation error change as TCE expands coverage, and how far can coverage be expanded before the error outweighs its benefit?

\vspace{-0.05in}
\paragraph{Analysis of the Theoretical Gap.}
To address the first question, we empirically evaluate the value-discrepancy formulation of the performance gap in Theorem~\ref{thm:gap} under morphology shifts, using a reference target value function trained on additional target data solely for evaluation. We consider Src(NN), consisting of source transitions nearest to the target data, along with TCE(Gen), TCE(Mix), OTDF, and MOBODY. As shown in Fig.~\ref{fig:coverage_reliability}(\subref{fig:gapbound}), Src(NN) exhibits a large value discrepancy despite its proximity to the target data, indicating that state proximity does not necessarily imply a small dynamics gap. In contrast, TCE(Gen) achieves a substantially smaller discrepancy, supporting its advantage over direct source reuse when the dynamics gap is large relative to the generation error. TCE(Mix) exhibits a larger discrepancy due to the additional use of source transitions, while OTDF and MOBODY similarly show larger discrepancies. These results demonstrate that our theoretical analysis captures the relative reliability of different ways of utilizing source information and explains why target-aligned generation is particularly effective under large dynamics gaps.

\vspace{-0.05in}
\paragraph{Coverage--Generation Error Trade-off.}
While reducing the discrepancy introduced by source utilization is beneficial, minimizing discrepancy alone is insufficient when target coverage is limited. We therefore examine how coverage expansion affects generation error and policy performance. As shown in Fig.~\ref{fig:coverage_reliability}(\subref{fig:analysis_sample}), increasing $\eta$ progressively expands generated samples from the limited target-state region toward the broader source-state region. Table~\ref{fig:error_table} shows that this expansion gradually increases action, reward, and transition errors relative to ground-truth target transitions obtained from the target simulator, along with the resulting value discrepancy. We also compare TCE(Gen) with TCE(Full-Trans.), which directly generates $(a,r,s')$ given $s$. As discussed in Section~\ref{subsec:sdes_TCE}, full-transition generation generalizes less effectively under scarce target data, resulting in larger generation errors and lower policy performance than our state-pair-first design.

With $\eta=0$, generation achieves the smallest errors and value discrepancy but remains confined to the limited target coverage. Moderate expansion improves performance despite increased generation error and value discrepancy, indicating that the benefit of broader coverage can outweigh the increased error. With excessive expansion, however, the increased generation error outweighs the benefit of additional coverage, causing performance to decline. These results reveal a clear trade-off between coverage expansion and generation error, highlighting the importance of expanding beyond limited target support while keeping the resulting generation error sufficiently small.

\begin{table*}[t]
    \vspace{-0.1in}
    \centering
    \caption{Generation errors, value discrepancies, and average normalized scores across coverage coefficients $\eta$ on 36 MuJoCo morphology-shift tasks.}
    \vspace{-0.1in}
    \label{fig:error_table}
    \setlength{\tabcolsep}{4.5pt}
    \renewcommand{\arraystretch}{1.12}
    \small
    \begin{tabular}{@{}l c c c c c c@{}}
    \toprule
    \multirow[c]{2}{*}[-0.4ex]{Model} & \multirow[c]{2}{*}[-0.4ex]{$\eta$} & \multicolumn{3}{c}{Error} & \multirow[c]{2}{*}[-0.4ex]{{\shortstack{Value\\Discrepancy}}} & \multirow[c]{2}{*}[-0.4ex]{\shortstack{Average\\Score}} \\
    \cmidrule(lr){3-5}
    & & Action & Reward & Trans. & & \\
    \midrule
    \multirow{5}{*}{TCE(Gen)}
    & $0$ & \textbf{0.038$\pm$0.009} & \textbf{0.024$\pm$0.010} & \textbf{0.016$\pm$0.005} & \textbf{1.47$\pm$0.77} & 54.6$\pm$25.1 \\
    & $0.1$ & 0.062$\pm$0.011 & 0.025$\pm$0.012 & 0.024$\pm$0.013 & {2.56$\pm$0.82} & 55.8$\pm$24.1 \\
    & $0.2$ & 0.066$\pm$0.014 & 0.025$\pm$0.013 & 0.029$\pm$0.012 & {3.01$\pm$0.89} & \textbf{57.4$\pm$24.6} \\
    & $0.5$ & 0.085$\pm$0.021 & 0.052$\pm$0.014 & 0.046$\pm$0.015 & {3.71$\pm$1.15} & 52.1$\pm$25.5 \\
    & $0.9$ & 0.118$\pm$0.020 & 0.069$\pm$0.017 & 0.052$\pm$0.019 & {3.85$\pm$0.97} & 50.9$\pm$26.1 \\
    \midrule
    TCE(Full-Trans.) & $0.2$ & 0.180$\pm$0.045 & 0.105$\pm$0.015 & 0.159$\pm$0.052 & {8.22$\pm$1.50} & 46.5$\pm$22.6 \\
    \bottomrule
    \end{tabular}%
    \vspace{-0.1in}
\end{table*}

{
    
\vspace{-0.1in}
\subsection{Further Ablation and Analysis}
\vspace{-0.1in}
\label{subsec:abl}
\sbox{\compEvalTableBox}{%
    \small
    \setlength{\tabcolsep}{3pt}%
    \begin{tabular}{@{}l|c@{}}
        \toprule
        \multicolumn{1}{c|}{\textbf{Setting}} & \textbf{Avg. Score} \\
        \midrule
        TCE(Gen) & \textbf{57.4$\pm$24.6} \\
        TCE(Gen-OT) & 55.6$\pm$24.2 \\
        TCE(Mix) & 48.1$\pm$25.5 \\
        TCE(Mix-OT) & {47.5$\pm$25.1} \\
        TCE(w/o Reg.) & 55.8$\pm$26.8 \\
        IQL* & {22.2$\pm$11.9} \\
        \bottomrule
    \end{tabular}%
}
\noindent
\begin{minipage}[t]{\dimexpr\linewidth-\wd\compEvalTableBox-2.8em\relax}
    \vspace{0pt}
    \noindent\textbf{Component Evaluation.} 
    Table~\ref{tab:comp_eval} further evaluates the selection strategy and policy regularization under morphology shifts. We compare TCE(Gen-OT) and TCE(Mix-OT), which replace NN selection with optimal-transport (OT) selection~\citep{lyu2025cross}, and TCE(w/o Reg.), which removes policy regularization. The OT variants perform comparably to their NN-based counterparts, with the differences between NN and OT being smaller than those between TCE(Gen) and TCE(Mix). This indicates that the relative benefit of generation versus direct source reuse is more influential than the specific selection strategy in these settings. TCE(w/o Reg.) shows only a modest performance decrease, indicating a limited contribution from policy regularization.
\end{minipage}%
\hfill
\hspace*{-0.65em}%
\begin{minipage}[t]{\wd\compEvalTableBox}
    \vspace{0pt}
    \centering
    \captionsetup{type=table,font=small,skip=4pt}
    \captionof{table}{Component evaluation using normalized scores averaged over 36 morphology-shift tasks.}
    \label{tab:comp_eval}
    \usebox{\compEvalTableBox}
\end{minipage}%
\hspace*{0.65em}

\vspace{-0.05in}
\paragraph{Further Analysis.}
We additionally provide a computational comparison in Appendix~\ref{secapp:compu}, showing that TCE(Gen) maintains peak GPU memory usage comparable to the baselines despite additional preprocessing. We further analyze the sensitivity to $\eta$ and $\lambda$ across additional domain shifts in Appendix~\ref{subsecapp:eta_lambda_abl}, with results consistent with our main coverage--error and source--generation analyses. Appendix~\ref{subsecapp:target-size} examines the effect of target-data size, showing that additional target data reduces inverse-dynamics error, while performance gains diminish beyond 5K transitions. Finally, we analyze the effect of denoising steps in Appendix~\ref{subsecapp:deno-steps}, where TCE(Gen) exhibits only modest performance variation across the tested values. {Limitations and future directions are discussed in Appendix~\ref{secapp:limitations}.}

}

\vspace{-0.1in}
\section{Conclusion}
\vspace{-0.1in}

We presented Target-aligned Coverage Expansion (TCE), which leverages source states to expand the limited target coverage through target-aligned generation while controlling generation error. Our performance-gap analysis characterizes the interplay between the source--target dynamics gap and generation error, providing guidance for effective coverage expansion and source utilization. Experiments across diverse dynamics shifts demonstrate strong performance and validate this trade-off: moderate expansion improves performance despite increased generation error, while direct source reuse becomes beneficial under smaller dynamics gaps. These findings highlight the importance of balancing coverage expansion against generation error and demonstrate that source states can effectively guide target-aligned generation even when large dynamics gaps limit direct source reuse.

{
\subsection*{AI use statement}
Generative AI tools were used solely to improve the clarity and phrasing of the manuscript. All AI-assisted revisions were reviewed by the authors, who take full responsibility for the final content.
}

\subsection*{Ethics Statement}
This work does not involve human subjects or personally identifiable information, and all experiments were conducted in simulated environments using offline reinforcement learning benchmarks. Because the proposed method augments training data with generated transitions, applications in safety-critical domains require validation in the real environment before a learned policy is deployed.

\subsection*{Reproducibility Statement}
The assumptions underlying our theoretical results are stated in Section~\ref{sec:TCE}, and the complete proof is provided in Appendix~\ref{secapp:theorem-proof}. Implementation details and experimental settings are described in Appendices~\ref{secapp:imp} and~\ref{secapp:expdetail}, respectively. Our experiments build on the D4RL and ODRL benchmarks, and the dataset compositions and environment-specific construction procedures are documented in the appendix.

\bibliography{iclr2027_conference}
\bibliographystyle{iclr2027_conference}

\newpage
\appendix
\onecolumn

\section[Limitations and Future Work]{Limitations and Future Work}
\label{secapp:limitations}

Although TCE consistently outperforms baselines in our experiments, it incurs additional training and sampling costs. However, Appendix~\ref{secapp:compu} shows that its peak GPU memory usage remains comparable to the baselines. TCE also introduces two hyperparameters for coverage expansion and source--generation mixing, but our analysis provides guidance for their selection by considering the coverage--error trade-off and the source--target dynamics gap. In addition, while TCE is effective under the large dynamics gaps evaluated in this work, it does not yet address extreme domain differences where source and target state distributions have little overlap or their state or action spaces differ. Extending TCE to such settings remains an important direction for future work.


\section[Proof of Theorem~\ref{thm:gap}]{{Proof} of {Theorem}~\ref{thm:gap}}
\label{secapp:theorem-proof}
We begin by introducing the \textit{Telescoping Lemma} \citep{vgdf}, a fundamental result that decomposes the performance difference between two MDPs arising from their dynamic mismatch. This lemma serves as a cornerstone for our analysis.

\begin{lemma}[Telescoping Lemma, Lemma C.1 in \citet{vgdf}] \label{lem:telescoping}
Let $\mathcal{M}_1 := (\mathcal{S}, \mathcal{A}, P_1, r, \gamma)$ and $\mathcal{M}_2 := (\mathcal{S}, \mathcal{A}, P_2, r, \gamma)$ be two MDPs differing only in their transition dynamics $P_1$ and $P_2$. For any policy $\pi$, defining the value discrepancy term as
\[
G^{\pi}_{\mathcal{M}_1,\mathcal{M}_2}(s, a) := \mathbb{E}_{s'\sim P_1}\left[V^{\pi}_{\mathcal{M}_2}(s')\right] - \mathbb{E}_{s'\sim P_2}\left[V^{\pi}_{\mathcal{M}_2}(s')\right],
\]
allows us to express the performance gap as:
\[
J_{\mathcal{M}_1}(\pi) - J_{\mathcal{M}_2}(\pi) = \frac{\gamma}{1-\gamma}\mathbb{E}_{(s,a)\sim\rho^{\pi}_{\mathcal{M}_1}}\left[G^{\pi}_{\mathcal{M}_1,\mathcal{M}_2}(s, a)\right].
\]
\end{lemma}

The proof of Lemma~\ref{lem:telescoping} is provided in \citet{vgdf}. Building upon this result, we restate and prove {Theorem~\ref{thm:gap}}, which bounds the performance gap between the mixture and target domains {through transition and value discrepancies}.


\begin{restatetheorem}{Theorem~\ref{thm:gap}}
\itshape
Let $J_{\mathrm{mix}}(\pi)$ and $\rho_{\mathrm{mix}}^{\pi}$ denote the expected return and the state-action occupancy measure in the mixture MDP, respectively. Then, the performance gap between the mixture MDP and the target MDP can be bounded in terms of transition dynamics as follows:
{\small
\begin{equation}
J_{\mathrm{mix}}(\pi) - J_{\mathrm{tar}}(\pi) \leq \frac{2\gamma r_{\max}}{(1-\gamma)^2} \Big( \lambda\,\mathbb{E}_{\rho_{\mathrm{mix}}^{\pi}}\big[ D_{\mathrm{TV}}(P_{\mathrm{src}} \,\|\, P_{\mathrm{tar}}) \big] + (1-\lambda)\,\mathbb{E}_{\rho_{\mathrm{mix}}^{\pi}}\big[ D_{\mathrm{TV}}(\hat P_{\mathrm{tar}} \,\|\, P_{\mathrm{tar}}) \big] \Big),
\end{equation}
}
where $|r(s,a)|\leq r_{\max}$ and $D_{\mathrm{TV}}(P \,\|\, Q)$ denotes the total variation distance between $P$ and $Q$.
\end{restatetheorem}

\begin{proof}[Proof {of the total-variation bound}]
To simplify notation, let $P_{\mathrm{mix}}$, $P_{\mathrm{src}}$, and $P_{\mathrm{tar}}$ denote $P_{\mathrm{mix}}(\cdot\mid s,a)$, $P_{\mathrm{src}}(\cdot\mid s,a)$, and $P_{\mathrm{tar}}(\cdot\mid s,a)$ respectively. Similarly, $\rho^{\pi}_{\mathrm{mix}}$ denotes $\rho^{\pi}_{\mathrm{mix}}(s,a)$, and $V^{\pi}_{\mathrm{tar}}$ represents $V^{\pi}_{\mathrm{tar}}(s')$.

Applying Lemma~\ref{lem:telescoping} with $\mathcal{M}_1 = \mathrm{mix}$ and $\mathcal{M}_2 = \mathrm{tar}$, we have:
{\small
\allowdisplaybreaks[1]
\begin{align*}
J_{\mathrm{mix}}(\pi)-J_{\mathrm{tar}}(\pi)
&=
\frac{\gamma}{1-\gamma}
\mathbb{E}_{\rho^{\pi}_{\mathrm{mix}}}
\left[
\int P_{\mathrm{mix}} V^{\pi}_{\mathrm{tar}} ds'
-
\int P_{\mathrm{tar}} V^{\pi}_{\mathrm{tar}} ds'
\right] \qquad \text{(by Lemma~\ref{lem:telescoping})} \\[0.3em]
&\leq
\frac{\gamma}{1-\gamma}
\mathbb{E}_{\rho^{\pi}_{\mathrm{mix}}}
\left[
\int 
\left| P_{\mathrm{mix}}-P_{\mathrm{tar}} \right|
\cdot \left| V^{\pi}_{\mathrm{tar}} \right| ds'
\right] \qquad \text{(Triangle Inequality)} \\[0.3em]
&\leq
\frac{\gamma}{1-\gamma} \cdot \frac{r_\text{max}}{1-\gamma}
\mathbb{E}_{\rho^{\pi}_{\mathrm{mix}}}\left[\int \left| P_{\mathrm{mix}}-P_{\mathrm{tar}} \right| ds' \right]   \\[0.3em]
&=
\frac{2\gamma r_{\text{max}}}{(1-\gamma)^2}
\mathbb{E}_{\rho^{\pi}_{\mathrm{mix}}}\left[ \frac{1}{2} \int |P_{\mathrm{mix}}-P_{\mathrm{tar}}| ds' \right] \\[0.3em]
&=
\frac{2\gamma r_{\text{max}}}{(1-\gamma)^2}
\mathbb{E}_{\rho^{\pi}_{\mathrm{mix}}}\left[D_{\text{TV}}(P_{\mathrm{mix}} \| P_{\mathrm{tar}})  \right]   \\[0.3em]
&=
\frac{2\gamma r_{\text{max}}}{(1-\gamma)^2}
\mathbb{E}_{\rho^{\pi}_{\mathrm{mix}}}\left[D_{\text{TV}}((\lambda P_{\mathrm{src}} + (1-\lambda) \hat{P}_{\mathrm{tar}}) \| P_{\mathrm{tar}})  \right] \\[0.3em]
&\leq
\frac{2\gamma r_{\text{max}}}{(1-\gamma)^2}
\mathbb{E}_{\rho^{\pi}_{\mathrm{mix}}}\left[\lambda D_{\text{TV}}( P_{\mathrm{src}} \| P_{\mathrm{tar}}) + (1-\lambda)D_{\text{TV}}( \hat{P}_{\mathrm{tar}} \| P_{\mathrm{tar}})  \right]
\\[0.3em]
&=
\frac{2\gamma r_{\text{max}}}{(1-\gamma)^2}
\Big(
\lambda\,\mathbb{E}_{\rho^{\pi}_{\mathrm{mix}}}\big[D_{\text{TV}}(P_{\mathrm{src}}\|P_{\mathrm{tar}})\big]
+(1-\lambda)\,\mathbb{E}_{\rho^{\pi}_{\mathrm{mix}}}\big[D_{\text{TV}}(\hat P_{\mathrm{tar}}\|P_{\mathrm{tar}})\big]
\Big)
\end{align*}
}
where the inequality follows from the convexity of total variation distance and the final equality follows from linearity of expectation, yielding the desired bound.
\end{proof}

We next {restate and} prove the value-discrepancy bound in Theorem~{\ref{thm:gap}}.

{Letting $V_{\mathrm{tar}}^\pi$ denote the target-domain value function, the performance gap is bounded by the corresponding value discrepancies as
\begin{align*}
J_{\mathrm{mix}}(\pi)-J_{\mathrm{tar}}(\pi)
\leq \frac{\gamma}{1-\gamma}\Big(
&\lambda\,\mathbb{E}_{\rho_{\mathrm{mix}}^\pi}\!\left[
\left|
\mathbb{E}_{P_{\mathrm{src}}}[V_{\mathrm{tar}}^\pi(s^{\prime})]
-\mathbb{E}_{P_{\mathrm{tar}}}[V_{\mathrm{tar}}^\pi(s^{\prime})]
\right|
\right] \\
&+(1-\lambda)\,\mathbb{E}_{\rho_{\mathrm{mix}}^\pi}\!\left[
\left|
\mathbb{E}_{\hat P_{\mathrm{tar}}}[V_{\mathrm{tar}}^\pi(s^{\prime})]
-\mathbb{E}_{P_{\mathrm{tar}}}[V_{\mathrm{tar}}^\pi(s^{\prime})]
\right|
\right]
\Big).
\end{align*}
}

\begin{proof}[Proof of {the value-discrepancy bound}]
To simplify notation, let $P_{\mathrm{mix}}$, $P_{\mathrm{src}}$, and $P_{\mathrm{tar}}$ denote $P_{\mathrm{mix}}(\cdot\mid s,a)$, $P_{\mathrm{src}}(\cdot\mid s,a)$, and $P_{\mathrm{tar}}(\cdot\mid s,a)$ respectively. Similarly, $\rho^{\pi}_{\mathrm{mix}}$ denotes $\rho^{\pi}_{\mathrm{mix}}(s,a)$, and $V^{\pi}_{\mathrm{tar}}$ represents $V^{\pi}_{\mathrm{tar}}(s')$.

Applying Lemma~\ref{lem:telescoping} with $\mathcal{M}_1 = \mathrm{mix}$ and $\mathcal{M}_2 = \mathrm{tar}$, we have:
{\small
\allowdisplaybreaks[1]
\begin{align*}
J_{\mathrm{mix}}(\pi)-J_{\mathrm{tar}}(\pi)
&=
\frac{\gamma}{1-\gamma}
\mathbb{E}_{\rho^{\pi}_{\mathrm{mix}}}
\left[
\int 
P_{\mathrm{mix}} V^{\pi}_{\mathrm{tar}} ds'
-
\int
P_{\mathrm{tar}} V^{\pi}_{\mathrm{tar}} ds'
\right]
\qquad (\text{by Lemma~\ref{lem:telescoping}})\\[0.3em]
&=
\frac{\gamma}{1-\gamma}
\mathbb{E}_{\rho^{\pi}_{\mathrm{mix}}}
\left[
\mathbb{E}_{P_{\mathrm{mix}}}[V^{\pi}_{\mathrm{tar}}(s')]
-
\mathbb{E}_{P_{\mathrm{tar}}}[V^{\pi}_{\mathrm{tar}}(s')]
\right] \\[0.3em]
&\leq
\frac{\gamma}{1-\gamma}
\mathbb{E}_{\rho^{\pi}_{\mathrm{mix}}}
\left[\left|
\mathbb{E}_{P_{\mathrm{mix}}}[V^{\pi}_{\mathrm{tar}}(s')]
-
\mathbb{E}_{P_{\mathrm{tar}}}[V^{\pi}_{\mathrm{tar}}(s')]
\right|\right] \\[0.3em]
&=
\frac{\gamma}{1-\gamma}
\mathbb{E}_{\rho^{\pi}_{\mathrm{mix}}}
\bigg[\Big|
\lambda \mathbb{E}_{P_{\mathrm{src}}}[V^{\pi}_{\mathrm{tar}}(s')] 
+
(1-\lambda) \mathbb{E}_{\hat P_{\mathrm{tar}}}[V^{\pi}_{\mathrm{tar}}(s')]
-
\mathbb{E}_{P_{\mathrm{tar}}}[V^{\pi}_{\mathrm{tar}}(s')]
\Big|\bigg] \\[0.3em]
&=
\frac{\gamma}{1-\gamma}
\mathbb{E}_{\rho^{\pi}_{\mathrm{mix}}}
\bigg[\Big|
\lambda (\mathbb{E}_{P_{\mathrm{src}}}[V^{\pi}_{\mathrm{tar}}(s')] 
-
\mathbb{E}_{P_{\mathrm{tar}}}[V^{\pi}_{\mathrm{tar}}(s')]) \notag \\*
&\qquad\qquad\qquad\qquad
+
(1-\lambda) (\mathbb{E}_{\hat P_{\mathrm{tar}}}[V^{\pi}_{\mathrm{tar}}(s')]
-
\mathbb{E}_{P_{\mathrm{tar}}}[V^{\pi}_{\mathrm{tar}}(s')])
\Big|\bigg] \\[0.3em]
&\leq
\frac{\gamma}{1-\gamma}
\mathbb{E}_{\rho^{\pi}_{\mathrm{mix}}}
\bigg[
\lambda \big|\mathbb{E}_{P_{\mathrm{src}}}[V^{\pi}_{\mathrm{tar}}(s')] 
-
\mathbb{E}_{P_{\mathrm{tar}}}[V^{\pi}_{\mathrm{tar}}(s')]\big| \notag \\*
&\qquad\qquad\qquad
+
(1-\lambda) \big|
\mathbb{E}_{\hat P_{\mathrm{tar}}}[V^{\pi}_{\mathrm{tar}}(s')]
-
\mathbb{E}_{P_{\mathrm{tar}}}[V^{\pi}_{\mathrm{tar}}(s')]\big|
\bigg] \\[0.3em]
&=
\frac{\gamma}{1-\gamma}\Bigg(
\lambda\,\mathbb{E}_{\rho^{\pi}_{\mathrm{mix}}}\Big[
\big|\mathbb{E}_{P_{\mathrm{src}}}[V^{\pi}_{\mathrm{tar}}(s')]
-\mathbb{E}_{P_{\mathrm{tar}}}[V^{\pi}_{\mathrm{tar}}(s')]\big|
\Big] \\*
&\qquad\qquad
+(1-\lambda)\,\mathbb{E}_{\rho^{\pi}_{\mathrm{mix}}}\Big[
\big|\mathbb{E}_{\hat P_{\mathrm{tar}}}[V^{\pi}_{\mathrm{tar}}(s')]
-\mathbb{E}_{P_{\mathrm{tar}}}[V^{\pi}_{\mathrm{tar}}(s')]\big|
\Big]
\Bigg) \qquad
\end{align*}
}
where the inequality follows from the triangle inequality and the final equality follows from linearity of expectation, proving the claim.
\end{proof}

{Theorem~\ref{thm:gap} highlights that effective coverage expansion requires not only broadening the target support, but also controlling the resulting performance gap. This motivates expanding the support of target-aligned generation without inducing excessive generation error and using the source--target dynamics gap to guide direct source incorporation. When this gap is small, target-near source transitions can complement the generated data; when it is large, their direct use becomes less favorable, and source data are instead used primarily for coverage expansion.}

\clearpage
\newpage

\section{Detailed Implementation and Algorithm of TCE}
\label{secapp:imp}
\vspace{-0.1in}

This section summarizes core components of our proposed TCE method: {Appendix~\ref{subsecapp:redef}} provides the configuration of the score-based generative models. {Appendix~\ref{subsecapp:detailTCE}} details the auxiliary models, dataset construction, and the offline policy learning. {Appendix~\ref{subsecapp:impNS}} provides network architecture and configuration details supporting the overall implementation.

\vspace{-0.1in}
\subsection{Details of Score-based Generative Model with SDEs}
\label{subsecapp:redef}

Our generative framework follows the score-based modeling with SDEs described in Section~\ref{subsec:sdes}. We adopt the variance exploding (VE) SDE formulation with drift \(f(x, \tau)=0\) and diffusion coefficient \(g(\tau) = \sqrt{d(\sigma^2)/d\tau}\). In our implementation, we use a VP-style variance schedule while retaining the VE reverse-time dynamics and predictor-corrector sampling, which can be viewed as a parameterization choice within the score-SDE framework. The noise scale \(\sigma(\tau)\) is defined using a quadratic schedule for the exponent \(\mathcal{B}(\tau)\):
\begin{equation}
\sigma(\tau) = \sqrt{1 - \exp\left(-\mathcal{B}(\tau)\right)}, \qquad \mathcal{B}(\tau) = \alpha_{\min} \tau + \frac{1}{2} (\alpha_{\max} - \alpha_{\min}) \tau^2,
\end{equation}
where we fix \(\alpha_{\min}=0.1\) and \(\alpha_{\max}=20\) to ensure a smooth noise injection process. This choice is better suited to our low-dimensional, data-scarce offline RL setting than the canonical geometric VE schedule, as it gives smoother noise growth and more stable score estimation under limited data.

The score network is optimized via the denoising score matching loss (Eq.~(\ref{eq:baseloss}) in main text) using the weighting function \(\zeta(\tau) = \sigma(\tau)^2\). In the absence of conditional input \(c\), this formulation simplifies to an unconditional score network \(q_\theta(x^\tau, \tau)\). 

For sample generation, we employ the predictor-corrector (PC) scheme with \(K=500\) discretization steps. This method combines a numerical SDE solver (Predictor) with Langevin dynamics (Corrector) to improve robustness and sample quality. The predictor integrates the reverse SDE:
\begin{align}
\text{(Predictor)} \quad x^{k-1} = x^k - g(\tau^k)^2 q_\theta(x^k, \tau^k \mid c) \Delta \tau^k + g(\tau^k) \sqrt{|\Delta \tau^k|} \, \xi^k, \quad \xi^k \sim \mathcal{N}(0, I).
\end{align}
Following the predictor step, a Langevin corrector step refines the sample based on the score function:
\begin{align}
\text{(Corrector)} \quad x^{k-1} \leftarrow x^{k-1} + \epsilon^k q_\theta(x^{k-1}, \tau^{k-1} \mid c) + \sqrt{2\epsilon^k} \xi^{k}, \quad \xi^k \sim \mathcal{N}(0, I),
\end{align}
where $\epsilon^k$ is a step-size parameter for the corrector. This refinement helps correct errors accumulated during the numerical integration, yielding higher-quality transitions.

\vspace{-0.1in}
\subsection{Detailed Implementation of TCE}
\label{subsecapp:detailTCE}
\vspace{-0.1in}

In this section, we provide implementation details for the auxiliary models, the dataset construction, and the offline policy learning. The training and sampling for the score-based generative models follow the methodology described in Section~\ref{subsec:sdes_TCE}.

\vspace{-0.1in}
\paragraph{Auxiliary Models and Labeling.} 
Since the score-based models generate only state transitions \((\hat{s}_t, \hat{s}_{t+1})\), we require auxiliary models to recover the missing actions and rewards. We train an inverse dynamics model \(\mathrm{Inv}_{\psi}\) and a reward model \(\hat R_{\phi}\) using the mean squared error (MSE) loss:
\begin{equation}
\mathcal{L}_{\mathrm{aux}}(\psi,\phi) = 
\underbrace{
\mathbb{E}_{\mathcal{D}_{\mathrm{tar}}} \| \mathrm{Inv}_{\psi}(s_t, s_{t+1}) - a_t \|^2_2
}_{\textbf{\small inverse dynamics loss}} \ + \ %
\underbrace{
\mathbb{E}_{\mathcal{D}_{\mathrm{tar}} \cup \mathcal{D}_{\mathrm{src}}} \| \hat R_{\phi}(s_t, a_t,s_{t+1}) - r_t \|^2_2
}_{\textbf{\small reward loss}}.
\end{equation}
The inverse dynamics model is trained solely on \(\mathcal{D}_{\mathrm{tar}}\) to ensure consistency with target dynamics, while the reward model is trained on \(\mathcal{D}_{\mathrm{tar}} \cup \mathcal{D}_{\mathrm{src}}\) to leverage the shared reward function. We then utilize these models to label the generated state transitions \((\hat{s}_t, \hat{s}_{t+1})\) with actions and rewards, forming the fully labeled dataset \(\mathcal{D}_{\mathrm{gen}}^{\eta}\).

\vspace{-0.1in}
\paragraph{Dataset Construction and IQL Training.}
{
For a mixture coefficient $\lambda$, we retain the closest $\lambda|\mathcal D_{\mathrm{src}}|$ source transitions under the same $d_{\mathrm{NN}}$ ranking and generate $(1-\lambda)|\mathcal D_{\mathrm{src}}|$ target-aligned transitions. The resulting dataset is \(\mathcal{D}_{\mathrm{train}}^{\eta,\lambda}=\mathcal{D}_{\mathrm{tar}}\cup\mathcal{D}_{\mathrm{src}}^\lambda\cup\mathcal{D}_{\mathrm{gen}}^\eta\). TCE(Gen) uses $\lambda=0$, whereas TCE(Mix) uses $\lambda>0$; both therefore have the same total number of policy-training transitions.
}

For policy learning, we employ Implicit Q-Learning (IQL) \citep{iql} on \(\mathcal{D}_{\mathrm{train}}^{\eta,\lambda}\). For brevity, we write this dataset as \(\mathcal D_{\mathrm{train}}\) in the objectives below. The training objectives for the value function \(V_\varphi\), Q-function \(Q_\varphi\), and policy \(\pi_\omega\) are defined as follows:
\begin{align}
\mathcal{L}_V(\varphi) 
&= \mathbb{E}_{(s_t,a_t)\sim \mathcal{D}_{\mathrm{train}}}
\left[ L_2^{\tau_V}\big(Q_{\varphi{'}}(s_t,a_t) - V_{\varphi}(s_t)\big) \right], \\
\mathcal{L}_Q(\varphi) 
&= \mathbb{E}_{(s_t,a_t,r_t,s_{t+1})\sim \mathcal{D}_{\mathrm{train}}}
\left[\big(r_t+\gamma V_{\varphi^-}(s_{t+1}) - Q_{\varphi}(s_t,a_t)\big)^2\right], \\
\mathcal{L}_{\pi}(\omega) 
&= -\mathbb{E}_{(s_t,a_t)\sim \mathcal{D}_{\mathrm{train}}}
\Big[\exp\big(\kappa \cdot \mathrm{Adv}(s_t,a_t)\big)\,\log \pi_{\omega}(a_t|s_t)\Big] \notag \\
&\quad + \beta\, \mathbb{E}_{s_t\sim\mathcal{D}_{\mathrm{tar}}}
\Big[D_{\mathrm{KL}}\big(\hat{\pi}_{b}(\cdot|s_t)\,\|\,\pi_{\omega}(\cdot|s_t)\big)\Big],
\end{align}
where \(L_2^{\tau_V}(u)=|\tau_V-\mathds{1}[u<0]|\,u^2\) is the expectile loss with \(\tau_V=0.7\); \(\varphi'\) and \(\varphi^-\) denote the target network and stop-gradient parameters, respectively; \(\mathrm{Adv}(s_t,a_t) = Q_\varphi(s_t,a_t) - V_\varphi(s_t)\) is the advantage weighted by temperature \(\kappa=3\); and \(\hat{\pi}_{b}\) is the empirical behavior policy cloned from \(\mathcal{D}_{\mathrm{tar}}\), with \(\beta\) controlling the KL-divergence penalty.

\vspace{-0.1in}
\subsection{Network Architecture and Configurations}
\label{subsecapp:impNS}
\vspace{-0.1in}

We employ four core neural network architectures to model {coverage-expansion} scores, target-transition scores, inverse dynamics, and reward estimation. As summarized in Table~\ref{sectab:net_arch},
all networks leverage MLP backbones with residual connections to ensure effective gradient flow. Details of environment-specific state and action dimensions (\(d_s\) and \(d_a\)) are provided in Table~\ref{sectab:envdetail}.

{Coverage-expansion} score network (\(q_{\theta}^{\mathrm{cov}}\)) estimates the score of perturbed states conditioned on diffusion time. It takes the state \(s_t\) and time \(\tau\) as inputs. The time embedding from a 2-layer MLP is concatenated with \(s_t\) and passed through residual dense blocks to output a score in \(\mathbb{R}^{d_s}\).

\vspace{-0.05in}
Target-transition score network (\(q_{\theta}^\mathrm{tran}\)) predicts the score of the next state \(s_{t+1}\) conditioned on \(s_t\) and \(\tau\). Like the coverage-expansion score network, it embeds \(\tau\) and \(s_t\) with separate MLPs. These embeddings are concatenated with \(s_{t+1}\) and passed through residual blocks to output a score in \(\mathbb{R}^{d_s}\).

\vspace{-0.05in}
Inverse dynamics (\(\mathrm{Inv}_\psi\)) and {reward} models (\({\hat R_\phi}\)) map state transition pairs \([s_t, s_{t+1}]\) to the action space \(\mathbb{R}^{d_a}\) and scalar rewards, respectively. Both networks share the same architecture, starting with a dense layer followed by residual blocks.

\vspace{-0.05in}
\begin{table*}[h]
\centering
\captionsetup{skip=4pt}
\caption{Detailed architectures of the score and dynamics models.}
\label{sectab:net_arch}
{\footnotesize \textbf{Notation:} $\text{MLP}(x; [h_1, \dots, h_k])$ denotes an MLP with hidden sizes $h_i$ and SiLU activations. $\text{Linear}(d)$ is a fully connected layer with output dimension $d$. $[\text{ResBlock}(d)]_{\times N}$ denotes a stack of $N$ residual MLP blocks.\par}
\vspace{0.3em}
\small
\setlength{\tabcolsep}{4pt}
\renewcommand{\arraystretch}{1.0}
\setlength{\extrarowheight}{0.5pt} 
\begin{tabularx}{\textwidth}{@{}c|c|>{\raggedright\arraybackslash}X@{}}
\toprule
Network & Inputs & Architecture Specification \\
\midrule
\multirow{3}{*}{\shortstack{{Coverage-expansion} score \\ $(q_{\theta}^{\mathrm{cov}})$}}
 & \multirow{3}{*}{\shortstack{State $s_t$ \\ Time $\tau$}} 
 & $h_\tau = \text{MLP}(\tau; [128, 128])$ \\
 & & $h_{\text{in}} = \text{Concat}[s_t, h_\tau]$ \\
 & & $\text{MLP}(h_{\text{in}}; [256]) \to [\text{ResBlock}(256)]_{\times 4} \to \text{Linear}(d_s)$ \\
\midrule
\multirow{4}{*}{\shortstack{Target-transition score \\ $(q_{\theta}^\mathrm{tran})$}}
 & \multirow{4}{*}{\shortstack{State $s_t$ \\ Time $\tau$ \\ Next state $s_{t+1}$}}
 & $h_\tau = \text{MLP}(\tau; [128, 128])$ \\
 & & $h_s = \text{MLP}(s_t; [128, 128])$ \\
 & & $h_{\text{in}} = \text{Concat}[s_{t+1}, h_\tau, h_s]$ \\
 & & $\text{MLP}(h_{\text{in}}; [256]) \to [\text{ResBlock}(256)]_{\times 4} \to \text{Linear}(d_s)$ \\
\midrule
\multirow{2}{*}{\shortstack{Inverse dynamics model \\ $(\mathrm{Inv}_{\psi})$}}
 & \multirow{2}{*}{\shortstack{State $s_t$ \\ Next state $s_{t+1}$}}
 & $h_{\text{in}} = \text{Concat}[s_t, s_{t+1}]$ \\
 & & $\text{MLP}(h_{\text{in}}; [256]) \to [\text{ResBlock}(256)]_{\times 3} \to \text{Linear}(d_a)$ \\
\midrule
\multirow{2}{*}{\shortstack{Reward model \\ $({\hat R_{\phi}})$}}
 & \multirow{2}{*}{\shortstack{State $s_t$ \\ Next state $s_{t+1}$}}
 & $h_{\text{in}} = \text{Concat}[s_t, s_{t+1}]$ \\
 & & $\text{MLP}(h_{\text{in}}; [256]) \to [\text{ResBlock}(256)]_{\times 3} \to \text{Linear}(1)$ \\
\bottomrule
\end{tabularx}
\end{table*}

\vspace{-0.05in}
\begin{table*}[ht]
\centering
\captionsetup{skip=4pt}
\caption{State dimension ($d_s$) and action dimension ($d_a$) of each agent.}
\small
\setlength{\tabcolsep}{4pt}
\renewcommand{\arraystretch}{1.0}
\begin{tabular}{@{}l|cc@{}}
\toprule
Agent Type & State Dimension ($d_s$) & Action Dimension ($d_a$) \\ 
\midrule
HalfCheetah, Walker2d & 17 & 6 \\
Hopper      & 11 & 3 \\
Ant         & 111 & 8 \\
Pen  & 45 & 24 \\
Door & 39 & 24 \\
\bottomrule
\end{tabular}
\label{sectab:envdetail}
\end{table*}

\clearpage
\newpage

\section{Detailed Experimental Setup}
\label{secapp:expdetail}

This section outlines the experimental setup in detail. It first summarizes the baselines used for comparison in {Appendix~\ref{subsecapp:baselines_explanation}}, then describes the environment and offline dataset configurations and domain shifts considered in {Appendix~\ref{subsecapp:environmental_setup}}, and finally presents the hyperparameter choices for model training and algorithm parameters in {Appendix~\ref{subsecapp:hyperparameter_setup}}.

\subsection{Baselines Explanation}
\label{subsecapp:baselines_explanation}

This section summarizes the baseline algorithms used for comparison with TCE approach.

\textbf{IQL{*}}~\citep{iql} {trains IQL on the union of source and target datasets, providing a baseline for direct source-data utilization without explicit dynamics alignment}. \textbf{Official Code:} \url{https://github.com/ikostrikov/implicit_q_learning}.

\textbf{DARA}~\citep{dara} mitigates the effect of dynamics mismatch by training domain classifiers on state-action-next-state and state-action pairs to quantify domain discrepancy. This discrepancy is used to adjust source rewards, encouraging source data to better align with target dynamics. For this method, we follow the implementation provided in \cite{lyu2025cross}.

\textbf{BOSA}~\citep{bosa} introduces supported value estimation to constrain critic updates to plausible transitions under target dynamics. The actor updates only consider supported actions, preventing exploitation of unsupported out-of-distribution transitions. For this method, we follow the implementation provided in \cite{lyu2025cross}.


\textbf{IGDF}~\citep{wen2024contrastive} learns cross-domain contrastive representations that distinguish source from target transitions. The resulting scores filter source data during critic training to ensure only reliable source transitions contribute. \textbf{Official Code:} \url{https://github.com/BattleWen/IGDF}.

\textbf{OTDF}~\citep{lyu2025cross} applies optimal transport to align source and target transitions by computing deviation scores and selectively weighting source samples. The policy update also integrates CVAE-based support regularization to ensure the learned policy remains consistent with the target action. \textbf{Official Code:} \url{https://github.com/dmksjfl/OTDF}.

\textbf{MOBODY}~\citep{mobody} is a model-based off-dynamics offline RL method that learns target dynamics from source and limited target data, then performs policy optimization via rollouts from the learned target dynamics. It models source and target with separate action encoders and a shared latent transition, and uses target-Q-weighted behavior cloning for policy regularization. \textbf{Official Code:} \url{https://github.com/guoyihonggyh/MOBODY-Model-Based-Off-Dynamics-Offline-Reinforcement-Learning}.

\clearpage
\newpage

\subsection{Environmental Setup}
\label{subsecapp:environmental_setup}

This section details the experimental setup used to evaluate our approach. We adopt the cross-domain continuous control setup, including environments, dataset compositions, and domain shift configurations, proposed by \citet{lyu2025cross}. The benchmark is based on MuJoCo environments \citep{todorov2012mujoco} and features four agent types: HalfCheetah, Hopper, Walker2d, and Ant. In addition, we evaluate Adroit dexterous-manipulation tasks from the ODRL benchmark \citep{lyu2024odrl}, {covering} higher-dimensional domains with more pronounced structural shifts.

\paragraph{Offline Datasets.}
The offline datasets consist of pre-collected data from both source and target domains, primarily drawn from the D4RL benchmark \citep{d4rl} for MuJoCo tasks and from the ODRL benchmark \citep{lyu2024odrl} for Adroit tasks.

\textbf{Source {Domain}.} For MuJoCo, source data comprises three quality levels: \texttt{medium}, \texttt{medium-replay}, and \texttt{medium-expert}. The \texttt{medium} dataset contains 1M samples generated by a partially trained SAC policy. The \texttt{medium-replay} dataset includes all samples in the SAC replay buffer up to the point of medium performance (approx. 0.2M to 0.4M samples), thus mixing low- to medium-quality experiences. Finally, the \texttt{medium-expert} datasets combine 50\% expert and 50\% suboptimal data, with total sizes of approximately 2M transitions. For Adroit, following ODRL, the source domain consists of expert-demonstration data with roughly 0.5M transitions per task.

\textbf{Target {Domain}.} The target domain is used to assess policy adaptation under distributional shifts. For MuJoCo, three types of dynamics changes are introduced to the agents: \textbf{Morphology} shift (modifying the agent's physical structure), \textbf{Kinematic} shift (restricting joint rotations to simulate malfunctions), and \textbf{Gravity} shift (altering gravitational acceleration). MuJoCo target datasets are limited to {at most 5K transitions per dataset} and are provided across \texttt{medium}, \texttt{medium-expert}, and \texttt{expert} quality levels. For Adroit, following ODRL, we evaluate two tasks, \texttt{Pen} and \texttt{Door}, under two target-domain shift types: \textbf{Morphology} shift (\texttt{shrink-finger}, reducing finger capsule sizes) and \textbf{Kinematic} shift (\texttt{broken-joint}, restricting finger joint ranges). Each target dataset is restricted to 5K transitions. The labels \texttt{M} and \texttt{H} denote medium and hard shift levels. Table~\ref{table:dataset_size} summarizes the dataset sizes by domain and quality.

\begin{table*}[ht]
\centering
\captionsetup{skip=4pt}
\caption{Dataset sizes by domain and data quality for MuJoCo and Adroit tasks.}
\small
\setlength{\tabcolsep}{4pt}
\renewcommand{\arraystretch}{1.0}
\begin{tabular}{l|lc|lccc}
\toprule
Environment & \multicolumn{2}{c|}{Source Domain} & \multicolumn{1}{c}{Target Domain} & Morphology & Kinematic & Gravity \\ 
\midrule
\multirow{3}{*}{HalfCheetah} & \texttt{medium}        & 1M & \texttt{medium}        & 5K & 5K & 5K \\
                             & \texttt{medium-replay} & 0.2M  & \texttt{medium-expert} & 5K & 5K & 5K \\
                             & \texttt{medium-expert} & 2M & \texttt{expert}        & 5K & 5K & 5K \\ \midrule
\multirow{3}{*}{Hopper}      & \texttt{medium}        & 1M & \texttt{medium}        & 5K & 5K & 5K \\
                             & \texttt{medium-replay} & 0.4M  & \texttt{medium-expert} & 4.3K & 5K & 4.3K \\
                             & \texttt{medium-expert} & 2M & \texttt{expert}        & 5K & 5K & 5K \\ \midrule
\multirow{3}{*}{Walker2d}    & \texttt{medium}        & 1M & \texttt{medium}        & 5K & 5K & 5K \\
                             & \texttt{medium-replay} & 0.3M  & \texttt{medium-expert} & 3.5K & 4.4K & 4.8K \\
                             & \texttt{medium-expert} & 2M & \texttt{expert}        & 5K & 5K & 5K \\ \midrule
\multirow{3}{*}{Ant}         & \texttt{medium}        & 1M & \texttt{medium}        & 5K & 5K & 5K \\
                             & \texttt{medium-replay} & 0.3M  & \texttt{medium-expert} & 5K & 5K & 3.1K \\
                             & \texttt{medium-expert} & 2M & \texttt{expert}        & 5K & 5K & 5K \\ \midrule
Pen            & \texttt{expert} & 0.5M & \texttt{expert} & 5K & 5K & -- \\ \midrule
Door           & \texttt{expert} & 0.5M & \texttt{expert} & 5K & 5K & -- \\
\bottomrule
\end{tabular}
\label{table:dataset_size}
\end{table*}

\paragraph{Evaluation Metric.}
All results are presented as normalized scores to fairly compare across environments with varying return scales:
\[
NS = \frac{J - J_r}{J_e - J_r} \times 100,
\]
where \(J\) is the return of the evaluated policy, and \(J_r\) and \(J_e\) represent returns from random and expert policies in the target domain, respectively. Reference scores proposed by  \citet{lyu2025cross} for each agent and target domain scenario are shown in Table~\ref{table:ref_score}. For Adroit, the morphology rows below correspond to the \texttt{shrink-finger} shift and the kinematic rows correspond to the \texttt{broken-joint} shift.

\begin{table*}[ht]
\centering
\captionsetup{skip=4pt}
\caption{Reference minimum score $J_r$ and maximum score $J_e$ by agent and domain shifts.}
\small
\setlength{\tabcolsep}{4pt}
\renewcommand{\arraystretch}{1.0}
\begin{tabular}{@{}ll|cc@{}}
\toprule
Agent Type & Domain Shifts & Reference {Min Score} $J_r$ & Reference {Max Score} $J_e$ \\ 
\midrule
\multirow{3}{*}{HalfCheetah} & Morphology & -280.18 & 9713.59 \\
                             & Kinematic  & -280.18 & 7065.03 \\
                             & Gravity    & -280.18 & 9509.15 \\ \midrule
\multirow{3}{*}{Hopper}      & Morphology & -26.34  & 3152.75 \\
                             & Kinematic  & -26.34  & 2842.73 \\
                             & Gravity    & -26.34  & 3234.3  \\ \midrule
\multirow{3}{*}{Walker2d}    & Morphology & 10.8    & 4398.43 \\
                             & Kinematic  & 10.8    & 3257.51 \\
                             & Gravity    & 10.8    & 5154.71 \\ \midrule
\multirow{3}{*}{Ant}         & Morphology & -325.6  & 5722.01 \\
                             & Kinematic  & -325.6  & 5122.57 \\
                             & Gravity    & -325.6  & 4317.07 \\ \midrule
\multirow{2}{*}{Pen} & Morphology & -12.17 & 6408.38 \\
 & Kinematic & -12.17 & 6408.38 \\ \midrule
\multirow{2}{*}{Door} & Morphology & -52.34 & 2880.57 \\
 & Kinematic & -52.34 & 2880.57 \\ 
\bottomrule
\end{tabular}
\label{table:ref_score}
\end{table*}

\paragraph{Domain Shifts.}

To simulate realistic but significant distributional differences between source and target domains, we consider separate domain-shift settings for MuJoCo and Adroit. These shifts directly impact the agent’s control dynamics and pose challenges to policy generalization. Fig.~\ref{fig:env_diff} provides visual examples of the source domain and the corresponding morphology-shifted target domains, highlighting the structural changes induced by these perturbations.

\begin{figure}[ht]
    \centering
    \includegraphics[width=0.95\textwidth]{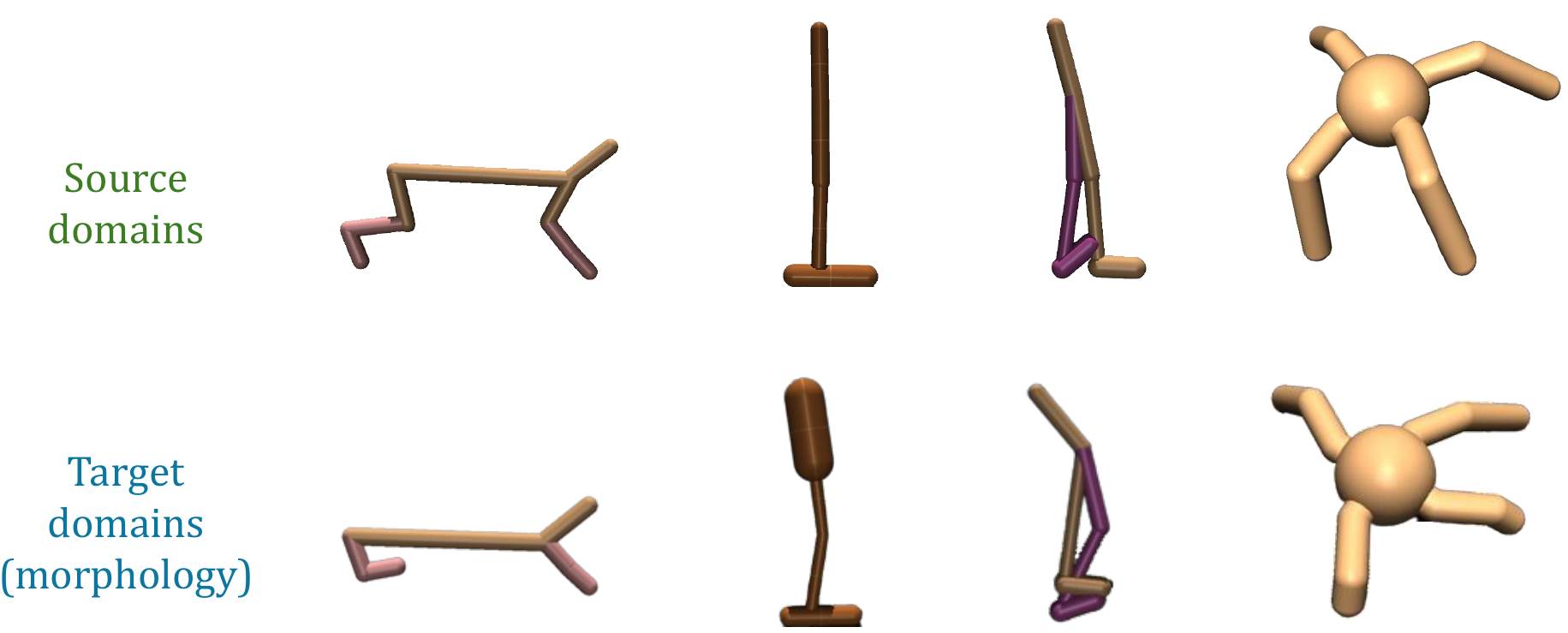}
    \caption{Visual examples of the source domain and morphology-shifted target domains in MuJoCo.}
    \label{fig:env_diff}
\end{figure}

For MuJoCo tasks, we apply three types of domain shifts:
\begin{itemize}[leftmargin=*]
    \item \textit{Morphology shift} involves structural modifications to the agent's body by altering body-segment dimensions. This includes shortening the thighs in HalfCheetah, {shortening} the torso in Hopper, {shortening the thigh and elongating the lower leg} in Walker2d, and reducing the front ankle capsules in Ant. These modifications create clear physical domain gaps that affect locomotion efficiency and balance.
    \item \textit{Kinematic shift} emulates impairments or restrictions in joint mobility by narrowing the allowed range of joint rotations. This corresponds to restricting the back thigh in HalfCheetah, the {thigh} and foot joints in Hopper, the foot joint in Walker2d, and the front hip joints in Ant, thereby reducing maneuverability in the target domain.
    \item \textit{Gravity shift} reduces the gravitational force magnitude acting on the agents by changing the target-domain gravity from the standard Earth value of -9.81 m/s\textsuperscript{2} in the source to -4.905 m/s\textsuperscript{2}. This creates a lighter environment in which the reduced downward force alters the agents' balance, contact dynamics, and overall locomotion efficiency during movement.
\end{itemize}

For Adroit tasks, following ODRL, we consider two target-domain shift types:
\begin{itemize}[leftmargin=*]
    \item \textit{Morphology shift} follows the \texttt{shrink-finger} setting, which reduces the capsule half-lengths of the finger phalanges, yielding shorter fingers with less contact area for stable grasping.
    \item \textit{Kinematic shift} follows the \texttt{broken-joint} setting, which contracts the rotational ranges of the index-finger and thumb joints, substantially limiting finger dexterity and precise grasp formation.
\end{itemize}

These domain shifts are implemented by modifying the MuJoCo XML configuration files following \citet{lyu2025cross} for MuJoCo and the ODRL XML and environment configurations for Adroit. Together, these settings enable a comprehensive evaluation of cross-domain offline reinforcement learning under varied and challenging environment changes.

{HalfCheetah morphology shift} : thigh length reduced by approximately 99\%
\begin{lstlisting}[style=xmlbox]
<geom fromto="0 0 0 -0.0001 0 -0.0001" name="bthigh" size="0.046" type="capsule"/>
<body name="bshin" pos="-0.0001 0 -0.0001">
<geom fromto="0 0 0 0.0001 0 0.0001" name="fthigh" size="0.046" type="capsule"/>
<body name="fshin" pos="0.0001 0 0.0001">
\end{lstlisting}

{Hopper morphology shift} : torso length reduced by approximately 56\%
\begin{lstlisting}[style=xmlbox]
<geom friction="0.9" fromto="0 0 1.45 0 0 1.05" name="torso_geom" size="0.125" type="capsule"/>
\end{lstlisting}

{Walker2d morphology shift} : thigh shortened and lower leg elongated by approximately 98\%
\begin{lstlisting}[style=xmlbox]
<body name="thigh" pos="0 0 1.05">
  <joint axis="0 -1 0" name="thigh_joint" pos="0 0 1.05" range="-150 0" type="hinge"/>
  <geom friction="0.9" fromto="0 0 1.05 0 0 1.045" name="thigh_geom" size="0.05" type="capsule"/>
  <body name="leg" pos="0 0 0.35">
    <joint axis="0 -1 0" name="leg_joint" pos="0 0 1.045" range="-150 0" type="hinge"/>
    <geom friction="0.9" fromto="0 0 1.045 0 0 0.3" name="leg_geom" size="0.04" type="capsule"/>
    <body name="foot" pos="0.2 0 0">
      <joint axis="0 -1 0" name="foot_joint" pos="0 0 0.3" range="-45 45" type="hinge"/>
      <geom friction="0.9" fromto="-0.0 0 0.3 0.2 0 0.3" name="foot_geom" size="0.06" type="capsule"/>
    </body>
  </body>
</body>
\end{lstlisting}

{Ant morphology shift} : front ankle capsule size reduced by 20\%
\begin{lstlisting}[style=xmlbox]
<geom fromto="0.0 0.0 0.0 0.1 0.1 0.0" name="left_ankle_geom" size="0.08" type="capsule"/>
<geom fromto="0.0 0.0 0.0 -0.1 0.1 0.0" name="right_ankle_geom" size="0.08" type="capsule"/>
\end{lstlisting}
 
{HalfCheetah kinematic shift} : back thigh joint range reduced by approximately 99\%
\begin{lstlisting}[style=xmlbox]
<joint axis="0 1 0" damping="6" name="bthigh" pos="0 0 0" range="-.0052 .0105" stiffness="240" type="hinge"/>
\end{lstlisting}

{Hopper kinematic shift} : thigh joint range reduced by 90\%, foot joint range reduced by 60\%
\begin{lstlisting}[style=xmlbox]
<joint axis="0 -1 0" name="thigh_joint" pos="0 0 1.05" range="-0.15 0" type="hinge"/>
<joint axis="0 -1 0" name="foot_joint" pos="0 0 0.1" range="-18 18" type="hinge"/>
\end{lstlisting}

{Walker2d kinematic shift} : foot joint range reduced by 99\%
\begin{lstlisting}[style=xmlbox]
<joint axis="0 -1 0" name="foot_joint" pos="0 0 0.1" range="-0.45 0.45" type="hinge"/>
\end{lstlisting}

{Ant kinematic shift} : front hip joint ranges reduced by approximately 89\%
\begin{lstlisting}[style=xmlbox]
<joint axis="0 0 1" name="hip_1" pos="0.0 0.0 0.0" range="-0.3 0.3" type="hinge"/>
<joint axis="0 0 1" name="hip_2" pos="0.0 0.0 0.0" range="-0.3 0.3" type="hinge"/>
\end{lstlisting}

{Gravity shift} : gravity magnitude reduced by 50\%

\begin{lstlisting}[style=xmlbox]
<option gravity="0 0 -4.905" timestep="0.01"/>
\end{lstlisting}

{Adroit kinematic shift (medium)} : restricted index finger and thumb joints by 75\%
\begin{lstlisting}[style=xmlbox]
<joint name="FFJ3" pos="0 0 0" axis="0 1 0" range="-0.109 0.109" user="1103"/>
<joint name="FFJ2" pos="0 0 0" axis="1 0 0" range="0 0.39275" user="1102"/>
<joint name="THJ4" pos="0 0 0" axis="0 0 -1" range="-0.26175 0.26175" user="1121"/>
<joint name="THJ0" pos="0 0 0" axis="0 1 0" range="-0.3925 0" user="1117"/>
\end{lstlisting}

{Adroit kinematic shift (hard)} : restricted index finger and thumb joints by 87.5\%
\begin{lstlisting}[style=xmlbox]
<joint name="FFJ3" pos="0 0 0" axis="0 1 0" range="-0.0545 0.0545" user="1103"/>
<joint name="FFJ2" pos="0 0 0" axis="1 0 0" range="0 0.196375" user="1102"/>
<joint name="THJ4" pos="0 0 0" axis="0 0 -1" range="-0.130875 0.130875" user="1121"/>
<joint name="THJ0" pos="0 0 0" axis="0 1 0" range="-0.19625 0" user="1117"/>
\end{lstlisting}

{Adroit morphology shift (medium)} : finger phalanx half-lengths reduced to 50\%
\begin{lstlisting}[style=xmlbox]
<geom name="C_ffproximal" class="DC_Hand" size="0.01 0.005625" pos="0 0 0.005625" type="capsule"/>
<geom name="C_ffmiddle"   class="DC_Hand" size="0.00805 0.003125" pos="0 0 0.003125" type="capsule"/>
<geom name="C_ffdistal"   class="DC_Hand" size="0.00705 0.003" pos="0 0 0.003" type="capsule" condim="4"/>
\end{lstlisting}

{Adroit morphology shift (hard)} : finger phalanx half-lengths reduced to 25\%
\begin{lstlisting}[style=xmlbox]
<geom name="C_ffproximal" class="DC_Hand" size="0.01 0.0028125" pos="0 0 0.0028125" type="capsule"/>
<geom name="C_ffmiddle"   class="DC_Hand" size="0.00805 0.0015625" pos="0 0 0.0015625" type="capsule"/>
<geom name="C_ffdistal"   class="DC_Hand" size="0.00705 0.0015" pos="0 0 0.0015" type="capsule" condim="4"/>
\end{lstlisting}

\newpage

\subsection{Hyperparameter Setup}
\label{subsecapp:hyperparameter_setup}

{Table~\ref{table:shared_hyperparameters} summarizes the shared settings for model training, data generation, and policy learning. We use $\eta=0.2$ and $\beta=0.001$ as the default coverage and policy regularization coefficients, respectively, including for Adroit tasks. For HalfCheetah and Walker2d under gravity shifts, we reduce $\eta$ to $0.1$ to limit expansion-induced generation error, as their target datasets often have lower quality or narrower distributions. For Walker2d under morphology shifts, we increase $\beta$ to $0.5$ to encourage stronger adherence to the target behavior policy under large structural gaps. Unless otherwise specified, TCE(Gen) and TCE(Mix) share the same environment-specific settings for $\eta$ and $\beta$. The mixture coefficient $\lambda$ is fixed at $0$ for TCE(Gen). For TCE(Mix), we evaluate $\lambda\in\{0.1,0.5,0.9\}$ and report the best result for each experimental setting. Table~\ref{table:algorithm_hyperparameters} summarizes the resulting settings for $\eta$ and $\beta$ and the selected $\lambda$ values for TCE(Mix).}

\begin{table*}[ht]
\centering
\captionsetup{skip=4pt}
\caption{Shared hyperparameters for model training, sampling, and RL.}
\small
\setlength{\tabcolsep}{4pt}
\renewcommand{\arraystretch}{1.0}
\begin{tabular}{@{}ll|c@{}}
\toprule
Category & Parameter Name & Value \\ 
\midrule
\multirow{8}{*}{Model Training} 
 & Learning rate & 1e-4 \\
 & Optimizer & Adam \\
 & Batch size & 128 \\
 & Training epochs for \(q_\theta^{\mathrm{cov}}\) & 10K \\
 & Training epochs for \(q_\theta^\mathrm{tran}\) & 5K \\
 & Training epochs for \(\mathrm{Inv}_\psi\), \(\hat R_\phi\) & 1K \\
 & Noise schedule \(\alpha_{\mathrm{min}}\) & 0.1 \\
 & Noise schedule \(\alpha_{\mathrm{max}}\) & 20 \\ \midrule
Sampling & Denoising steps \(K\) & 500 \\ \midrule
\multirow{5}{*}{RL Training} 
 & Learning rate for {actor, critic} & 3e-4 \\
 & Optimizer for {actor, critic} & Adam \\
 & Batch size for target sample & 128 \\
 & Batch size for generated sample & 128 \\
 & Training steps & 1M \\ 
\bottomrule
\end{tabular}
\label{table:shared_hyperparameters}
\end{table*}

\begin{table*}[ht]
\centering
\captionsetup{skip=4pt}
\caption{{Hyperparameter settings used for the reported results.}}
\small
\setlength{\tabcolsep}{4pt}
\renewcommand{\arraystretch}{1.0}
\begin{tabular}{@{}cl|c@{}}
\toprule
Parameter & Environment & Value \\ 
\midrule
\multirow{2}{*}{$\eta$} & HalfCheetah, Walker2d (gravity shift) & 0.1 \\ 
 & Other environments & 0.2 \\ \midrule
\multirow{2}{*}{$\lambda$} & {HalfCheetah (gravity shift)} & {0.9} \\
 & {Other environments} & 0.1 \\ \midrule
\multirow{2}{*}{$\beta$} & Walker2d (morphology shift)  & 0.5 \\ 
 & Other environments & 0.001 \\ 
\bottomrule
\end{tabular}
\label{table:algorithm_hyperparameters}
\end{table*}

\clearpage
\section{Additional Performance Comparison}
\label{secapp:perf_comp}
This section complements the main-text evaluation with full performance comparisons under kinematic and gravity shifts in Appendices~\ref{subsecapp:kinematic_shift} and~\ref{subsecapp:gravity_shift}. In Appendix~\ref{subsecapp:gravity_sweep}, we further examine how the benefit of source reuse changes across a wider range of gravity shifts. We broaden the evaluation through comparisons with additional baselines in Appendix~\ref{subsecapp:meta_comp}.

\subsection{Performance Comparison under Kinematic Shifts}
\label{subsecapp:kinematic_shift}

Table~\ref{table:perform_kinematic} shows that both TCE variants outperform all baselines in aggregate, with TCE(Gen) achieving the highest total score.
By using source states to expand the support of target-aligned generation, TCE can benefit from broader state coverage even when joint-range restrictions make source transitions poorly aligned with target dynamics.
The overall advantage of TCE(Gen) is consistent with this distinction, while the gains from TCE(Mix) on some tasks suggest that additionally incorporating target-near source transitions can still complement the generated data.
TCE(Mix) reports $\lambda=0.1$, the best evaluated positive mixture coefficient for every task.

\begin{table}[H]
\centering
\captionsetup{skip=4pt}
\caption{Performance comparison on kinematic-shift tasks. Abbrev.: half=HalfCheetah, hopp=Hopper, walk=Walker2d, ant=Ant; m=\texttt{medium}, m-r=\texttt{medium-replay}, e=\texttt{expert}, m-e=\texttt{medium-expert}. Src./Tgt. denote source/target domains. Results are reported as mean $\pm$ standard deviation over 5 seeds; best results are in \textbf{bold} and second-best results are \underline{underlined}.}
\label{table:perform_kinematic}

\small
\setlength{\tabcolsep}{4pt}
\renewcommand{\arraystretch}{1.0}
\resizebox{\textwidth}{!}{%
\begin{tabular}{@{}S{1.25cm}l|cccccc|cc@{}}
\toprule
Src. & Tgt. & IQL* & DARA & BOSA & IGDF & MOBODY & OTDF & TCE(Gen) & TCE(Mix) \\
\midrule
half-m & m & 12.3$\pm$1.2 & 10.6$\pm$1.2 & 8.3$\pm$1.2 & 23.6$\pm$5.7 & 35.4$\pm$4.8 & 40.2$\pm$0.0 & \textbf{42.2$\pm$2.4} & \underline{41.6$\pm$1.8} \\
half-m & m-e & 10.8$\pm$1.9 & 12.9$\pm$2.8 & 8.7$\pm$1.3 & 9.8$\pm$2.4 & 19.5$\pm$3.2 & 10.1$\pm$4.0 & \textbf{40.7$\pm$2.7} & \underline{40.1$\pm$2.0} \\
half-m & e & 12.6$\pm$1.7 & 12.1$\pm$1.0 & 10.8$\pm$1.7 & \underline{12.8$\pm$0.7} & 3.9$\pm$1.6 & 8.7$\pm$2.0 & \textbf{21.3$\pm$0.8} & 6.7$\pm$2.1 \\
half-m-r & m & 10.0$\pm$5.4 & 11.5$\pm$4.9 & 7.5$\pm$3.1 & 11.6$\pm$4.6 & 14.5$\pm$4.9 & 37.8$\pm$2.1 & \textbf{42.0$\pm$0.2} & \underline{41.7$\pm$1.8} \\
half-m-r & m-e & 6.5$\pm$3.1 & 9.2$\pm$4.7 & 6.6$\pm$1.7 & 8.6$\pm$2.3 & 4.1$\pm$0.7 & 9.7$\pm$2.0 & \textbf{41.6$\pm$0.5} & \underline{37.8$\pm$2.8} \\
half-m-r & e & 13.6$\pm$1.4 & \underline{14.8$\pm$2.0} & 10.4$\pm$4.9 & 13.9$\pm$2.2 & 4.1$\pm$0.9 & 7.2$\pm$1.4 & \textbf{23.4$\pm$1.5} & 4.3$\pm$0.3 \\
half-m-e & m & 21.8$\pm$6.5 & 25.9$\pm$7.4 & 30.0$\pm$4.3 & 21.9$\pm$6.5 & 30.2$\pm$5.3 & 30.7$\pm$9.6 & \textbf{41.9$\pm$0.3} & \underline{41.7$\pm$0.1} \\
half-m-e & m-e & 7.6$\pm$1.4 & 9.5$\pm$4.2 & 6.8$\pm$2.9 & 8.9$\pm$3.3 & 7.4$\pm$3.3 & 10.9$\pm$4.2 & \textbf{40.0$\pm$2.5} & \underline{39.8$\pm$3.3} \\
half-m-e & e & 9.1$\pm$2.4 & 10.4$\pm$1.3 & 4.9$\pm$3.2 & \underline{10.7$\pm$1.4} & 3.3$\pm$1.2 & 3.2$\pm$0.6 & \textbf{11.2$\pm$1.2} & 4.3$\pm$0.3 \\ \midrule
hopp-m & m & 58.7$\pm$8.4 & 43.9$\pm$15.2 & 12.3$\pm$6.6 & 65.3$\pm$1.4 & 63.3$\pm$3.7 & 65.6$\pm$1.9 & \textbf{66.0$\pm$0.4} & \underline{65.7$\pm$0.5} \\
hopp-m & m-e & 68.5$\pm$12.4 & 55.4$\pm$16.9 & 15.6$\pm$10.8 & 51.1$\pm$18.5 & 25.4$\pm$12.5 & 55.4$\pm$25.1 & \textbf{75.7$\pm$2.2} & \underline{70.3$\pm$10.0} \\
hopp-m & e & 79.9$\pm$35.5 & 83.7$\pm$19.6 & 14.8$\pm$5.5 & \underline{87.4$\pm$25.4} & 84.2$\pm$11.2 & 35.0$\pm$19.4 & \textbf{91.5$\pm$5.7} & 76.0$\pm$15.5 \\
hopp-m-r & m & 36.0$\pm$0.1 & 39.4$\pm$7.2 & 3.2$\pm$2.6 & 35.9$\pm$2.4 & 44.1$\pm$5.3 & 35.5$\pm$12.2 & \underline{65.9$\pm$1.4} & \textbf{66.2$\pm$0.2} \\
hopp-m-r & m-e & 36.1$\pm$0.1 & 34.1$\pm$3.6 & 4.4$\pm$2.8 & 36.1$\pm$0.1 & 36.1$\pm$4.8 & 47.5$\pm$14.6 & \textbf{82.8$\pm$4.3} & \underline{65.1$\pm$7.9} \\
hopp-m-r & e & 36.0$\pm$0.1 & 36.1$\pm$0.2 & 3.7$\pm$2.5 & 36.1$\pm$0.3 & \textbf{86.7$\pm$6.5} & 49.9$\pm$30.5 & \underline{84.8$\pm$7.0} & 72.3$\pm$16.1 \\
hopp-m-e & m & 66.0$\pm$0.5 & 61.1$\pm$4.0 & 35.0$\pm$20.1 & 65.2$\pm$1.5 & \underline{66.1$\pm$0.8} & 65.3$\pm$2.4 & \textbf{67.3$\pm$0.8} & 66.0$\pm$0.1 \\
hopp-m-e & m-e & 45.1$\pm$15.7 & 61.9$\pm$16.9 & 13.9$\pm$4.9 & 62.9$\pm$15.6 & 41.8$\pm$14.6 & 38.6$\pm$15.9 & \textbf{74.7$\pm$8.9} & \underline{72.1$\pm$6.8} \\
hopp-m-e & e & 44.9$\pm$19.8 & \underline{84.2$\pm$21.1} & 12.0$\pm$4.3 & 52.8$\pm$39.7 & \textbf{96.3$\pm$0.6} & 29.9$\pm$11.3 & 65.0$\pm$7.5 & 62.9$\pm$31.9 \\ \midrule
walk-m & m & 34.3$\pm$9.8 & 35.2$\pm$22.5 & 14.3$\pm$11.2 & 41.9$\pm$11.2 & 52.5$\pm$9.8 & 49.6$\pm$18.0 & \textbf{55.3$\pm$2.0} & \underline{53.5$\pm$7.7} \\
walk-m & m-e & 30.2$\pm$12.5 & \textbf{51.9$\pm$11.5} & 13.6$\pm$7.7 & 42.3$\pm$19.3 & 26.4$\pm$11.3 & \underline{43.5$\pm$16.4} & 35.9$\pm$10.2 & 18.1$\pm$10.2 \\
walk-m & e & 56.4$\pm$18.2 & 40.7$\pm$14.4 & 15.3$\pm$2.5 & \underline{60.4$\pm$17.5} & 53.8$\pm$12.5 & 46.7$\pm$13.6 & \textbf{62.8$\pm$1.7} & 38.7$\pm$15.3 \\
walk-m-r & m & 11.5$\pm$7.1 & 12.5$\pm$4.3 & 1.9$\pm$2.1 & 22.2$\pm$5.2 & 34.5$\pm$7.5 & \underline{49.7$\pm$9.7} & \textbf{55.6$\pm$9.5} & 43.1$\pm$4.8 \\
walk-m-r & m-e & 9.7$\pm$3.8 & 11.2$\pm$5.0 & 4.6$\pm$3.0 & 7.6$\pm$4.9 & 23.9$\pm$4.2 & \textbf{55.9$\pm$17.1} & 40.9$\pm$10.0 & \underline{47.2$\pm$10.3} \\
walk-m-r & e & 7.7$\pm$4.8 & 7.4$\pm$2.4 & 3.6$\pm$1.5 & 7.5$\pm$2.1 & 27.2$\pm$4.5 & \textbf{51.9$\pm$7.9} & \underline{49.9$\pm$7.4} & 38.5$\pm$9.3 \\
walk-m-e & m & 41.8$\pm$8.8 & 38.1$\pm$14.4 & 21.4$\pm$8.3 & 41.2$\pm$13.0 & 46.2$\pm$8.7 & 44.6$\pm$6.0 & \underline{46.3$\pm$5.4} & \textbf{47.1$\pm$10.9} \\
walk-m-e & m-e & 22.2$\pm$8.7 & 23.6$\pm$8.1 & 15.9$\pm$4.1 & 28.1$\pm$4.0 & 17.2$\pm$3.7 & 16.5$\pm$7.2 & \underline{33.6$\pm$4.3} & \textbf{38.3$\pm$10.5} \\
walk-m-e & e & 26.3$\pm$10.4 & 36.0$\pm$9.2 & 18.5$\pm$3.6 & \underline{46.2$\pm$19.4} & 42.1$\pm$9.4 & 42.4$\pm$9.1 & \textbf{54.6$\pm$4.4} & 42.5$\pm$5.7 \\ \midrule
ant-m & m & 50.0$\pm$5.6 & 42.3$\pm$7.6 & 20.9$\pm$2.6 & 54.5$\pm$1.3 & 51.5$\pm$8.7 & \textbf{55.4$\pm$0.0} & \underline{55.0$\pm$3.7} & 54.9$\pm$5.7 \\
ant-m & m-e & 57.8$\pm$7.2 & 54.1$\pm$3.8 & 31.7$\pm$7.0 & 54.5$\pm$4.6 & \underline{65.0$\pm$7.6} & 60.7$\pm$3.6 & \textbf{65.1$\pm$6.4} & 61.8$\pm$1.0 \\
ant-m & e & 59.6$\pm$18.5 & 54.2$\pm$11.3 & 45.4$\pm$8.6 & 49.4$\pm$14.6 & 80.9$\pm$6.8 & 90.4$\pm$4.8 & \underline{93.9$\pm$4.0} & \textbf{94.3$\pm$1.8} \\
ant-m-r & m & 43.7$\pm$4.6 & 42.0$\pm$5.4 & 19.0$\pm$1.8 & 41.4$\pm$5.0 & 49.2$\pm$6.4 & 52.8$\pm$4.4 & \textbf{55.6$\pm$9.1} & \underline{54.1$\pm$4.2} \\
ant-m-r & m-e & 36.5$\pm$5.9 & 36.0$\pm$6.7 & 19.1$\pm$1.6 & 37.2$\pm$4.7 & 53.7$\pm$5.9 & 54.2$\pm$5.2 & \textbf{65.6$\pm$5.4} & \underline{61.1$\pm$1.0} \\
ant-m-r & e & 24.4$\pm$4.8 & 22.1$\pm$0.4 & 19.5$\pm$0.8 & 24.3$\pm$2.8 & 39.8$\pm$7.4 & 74.7$\pm$10.5 & \underline{92.3$\pm$3.7} & \textbf{93.4$\pm$4.5} \\
ant-m-e & m & 49.5$\pm$4.1 & 44.7$\pm$4.3 & 19.0$\pm$8.0 & 41.8$\pm$8.8 & 46.5$\pm$4.5 & 50.2$\pm$4.3 & \textbf{57.2$\pm$2.7} & \underline{53.1$\pm$4.6} \\
ant-m-e & m-e & 37.2$\pm$2.0 & 33.3$\pm$7.0 & 6.4$\pm$2.5 & 41.5$\pm$4.9 & 49.0$\pm$8.1 & 48.8$\pm$2.7 & \textbf{69.0$\pm$1.8} & \underline{59.4$\pm$10.5} \\
ant-m-e & e & 18.7$\pm$8.1 & 17.8$\pm$23.6 & 14.5$\pm$9.0 & 14.4$\pm$22.9 & 63.6$\pm$5.5 & 78.4$\pm$12.2 & \textbf{97.3$\pm$1.3} & \underline{93.7$\pm$3.0} \\
\midrule
\multicolumn{2}{c|}{\textbf{Total Score}} & 1193.0 & 1219.8 & 513.5 & 1271.0 & 1489.4 & 1547.6 & \textbf{2063.9} & \underline{1867.4} \\
\bottomrule
\end{tabular}
}
\end{table}

\vspace{-0.1in}

\clearpage
\subsection{Performance Comparison under Gravity Shifts}
\label{subsecapp:gravity_shift}

In Table~\ref{table:perform_gravity}, TCE(Mix) achieves the highest total score when target gravity is reduced to 0.5$\times$ the standard value, although TCE(Gen) remains stronger on many individual tasks.
TCE(Mix) reports $\lambda=0.9$ for HalfCheetah and $\lambda=0.1$ for the other agents, the best evaluated mixture coefficients in these settings.
The substantial gains from source-heavy mixing in HalfCheetah suggest that selected source transitions can remain sufficiently aligned with target dynamics to benefit policy learning despite the change in gravity.
This interpretation aligns with Theorem~\ref{thm:gap}: a mild source--target dynamics gap limits the source-induced contribution to the performance-gap bound, supporting the use of target-near source transitions to complement target-aligned generation.

\begin{table}[H]
\centering
\captionsetup{skip=4pt}
\caption{Performance comparison on gravity-shift tasks. Abbrev.: half=HalfCheetah, hopp=Hopper, walk=Walker2d, ant=Ant; m=\texttt{medium}, m-r=\texttt{medium-replay}, e=\texttt{expert}, m-e=\texttt{medium-expert}. Src./Tgt. denote source/target domains. Results are reported as mean $\pm$ standard deviation over 5 seeds; best results are in \textbf{bold} and second-best results are \underline{underlined}.}
\label{table:perform_gravity}

\small
\setlength{\tabcolsep}{4pt}
\renewcommand{\arraystretch}{1.0}
\resizebox{\textwidth}{!}{%
\begin{tabular}{@{}S{1.25cm}l|cccccc|cc@{}}
\toprule
Src. & Tgt. & IQL* & DARA & BOSA & IGDF & MOBODY & OTDF & TCE(Gen) & TCE(Mix) \\
\midrule
half-m & m & 39.6$\pm$3.3 & \textbf{41.2$\pm$3.9} & 38.9$\pm$4.0 & 36.6$\pm$5.5 & 38.4$\pm$5.4 & \underline{40.7$\pm$7.7} & 4.1$\pm$0.4 & \underline{40.7$\pm$1.2} \\
half-m & m-e & 39.6$\pm$3.7 & \textbf{40.7$\pm$2.8} & \underline{40.4$\pm$3.0} & 38.7$\pm$6.2 & 35.4$\pm$4.6 & 28.6$\pm$3.2 & 3.1$\pm$0.3 & 39.2$\pm$0.5 \\
half-m & e & \textbf{42.4$\pm$3.8} & 39.8$\pm$4.4 & 40.5$\pm$3.9 & 39.6$\pm$4.6 & 37.1$\pm$5.5 & 36.1$\pm$5.3 & 14.8$\pm$3.9 & \underline{41.1$\pm$6.1} \\
half-m-r & m & \underline{20.1$\pm$5.0} & 17.6$\pm$6.2 & 20.0$\pm$4.9 & 14.4$\pm$2.2 & 17.9$\pm$6.8 & \textbf{21.5$\pm$6.5} & 4.0$\pm$1.8 & 18.1$\pm$7.2 \\
half-m-r & m-e & 17.2$\pm$1.6 & \underline{20.2$\pm$5.2} & 16.7$\pm$4.2 & 10.0$\pm$2.5 & 16.8$\pm$5.3 & 14.7$\pm$4.1 & 2.9$\pm$0.1 & \textbf{20.3$\pm$5.1} \\
half-m-r & e & \underline{20.7$\pm$5.5} & \textbf{22.4$\pm$1.7} & 15.4$\pm$4.2 & 15.3$\pm$3.7 & 19.5$\pm$4.5 & 11.4$\pm$1.9 & 7.9$\pm$1.8 & 15.5$\pm$1.3 \\
half-m-e & m & 38.6$\pm$6.0 & 37.8$\pm$3.3 & \textbf{41.8$\pm$5.1} & 37.7$\pm$7.3 & 38.2$\pm$7.6 & \underline{39.5$\pm$3.5} & 2.5$\pm$0.9 & 39.4$\pm$2.7 \\
half-m-e & m-e & 39.6$\pm$3.0 & 39.4$\pm$4.4 & 38.7$\pm$2.7 & \underline{40.7$\pm$3.2} & 34.4$\pm$7.8 & 32.4$\pm$5.5 & 2.8$\pm$0.1 & \textbf{41.3$\pm$4.5} \\
half-m-e & e & \underline{43.4$\pm$0.9} & \textbf{45.3$\pm$1.3} & 39.9$\pm$2.7 & 41.1$\pm$4.1 & 40.2$\pm$3.4 & 26.5$\pm$9.1 & 8.5$\pm$2.6 & 41.7$\pm$1.4 \\ \midrule
hopp-m & m & 11.2$\pm$1.1 & 17.3$\pm$3.8 & 15.2$\pm$3.3 & 15.3$\pm$3.5 & 20.1$\pm$4.5 & 32.4$\pm$8.0 & \textbf{61.9$\pm$2.7} & \underline{52.5$\pm$4.7} \\
hopp-m & m-e & 14.7$\pm$3.6 & 15.4$\pm$2.5 & 21.1$\pm$9.3 & 15.1$\pm$3.6 & 15.6$\pm$3.3 & 24.2$\pm$3.6 & \textbf{46.2$\pm$2.6} & \underline{35.0$\pm$7.9} \\
hopp-m & e & 12.5$\pm$1.6 & 19.3$\pm$10.5 & 12.7$\pm$1.7 & 14.4$\pm$0.8 & 14.3$\pm$3.2 & \underline{33.7$\pm$7.8} & \textbf{53.4$\pm$2.2} & 18.9$\pm$10.4 \\
hopp-m-r & m & 13.9$\pm$2.9 & 10.7$\pm$4.3 & 3.3$\pm$1.9 & 15.3$\pm$4.4 & 20.3$\pm$9.9 & 31.1$\pm$13.4 & \textbf{60.7$\pm$2.7} & \underline{58.3$\pm$0.8} \\
hopp-m-r & m-e & 13.3$\pm$6.3 & 12.5$\pm$5.6 & 4.6$\pm$1.7 & 15.4$\pm$5.5 & 19.5$\pm$4.2 & \underline{24.2$\pm$6.1} & \underline{24.2$\pm$3.1} & \textbf{32.7$\pm$10.4} \\
hopp-m-r & e & 11.0$\pm$2.6 & 14.3$\pm$6.0 & 3.2$\pm$0.8 & 16.1$\pm$4.0 & 17.5$\pm$2.8 & 31.0$\pm$9.8 & \textbf{73.4$\pm$7.8} & \underline{55.7$\pm$15.0} \\
hopp-m-e & m & 19.1$\pm$6.6 & 18.5$\pm$12.3 & 15.9$\pm$5.9 & 22.3$\pm$5.4 & 19.2$\pm$4.8 & 26.4$\pm$10.1 & \textbf{61.3$\pm$5.6} & \underline{54.0$\pm$5.8} \\
hopp-m-e & m-e & 16.8$\pm$2.7 & 16.0$\pm$6.1 & 17.3$\pm$2.5 & 16.6$\pm$7.7 & 17.9$\pm$3.5 & 28.3$\pm$6.7 & \textbf{45.5$\pm$9.8} & \underline{33.4$\pm$7.2} \\
hopp-m-e & e & 20.9$\pm$4.1 & 23.9$\pm$14.8 & 23.2$\pm$7.9 & 26.0$\pm$9.2 & 24.6$\pm$4.7 & 44.9$\pm$10.6 & \textbf{66.3$\pm$5.8} & \underline{55.8$\pm$8.3} \\ \midrule
walk-m & m & 28.1$\pm$12.9 & 28.4$\pm$13.7 & \textbf{38.0$\pm$11.2} & 22.1$\pm$8.4 & \underline{36.6$\pm$3.8} & \underline{36.6$\pm$2.3} & 36.3$\pm$0.9 & \underline{36.6$\pm$4.1} \\
walk-m & m-e & 35.7$\pm$4.7 & 30.7$\pm$9.7 & 40.9$\pm$7.2 & 35.4$\pm$9.1 & \underline{42.2$\pm$6.8} & \textbf{44.8$\pm$7.5} & 20.6$\pm$6.4 & 40.9$\pm$6.0 \\
walk-m & e & 37.3$\pm$8.0 & 36.0$\pm$7.0 & 41.3$\pm$8.6 & 36.2$\pm$13.6 & \textbf{49.9$\pm$7.9} & 44.0$\pm$4.0 & 20.3$\pm$1.6 & \underline{46.9$\pm$1.6} \\
walk-m-r & m & 14.6$\pm$2.5 & 14.1$\pm$6.1 & 7.6$\pm$5.8 & 11.6$\pm$4.6 & 20.8$\pm$6.5 & \underline{32.7$\pm$7.0} & \textbf{35.8$\pm$0.8} & 20.3$\pm$2.8 \\
walk-m-r & m-e & 15.3$\pm$1.9 & 15.9$\pm$5.8 & 4.8$\pm$5.8 & 13.9$\pm$6.5 & 20.1$\pm$9.6 & \textbf{31.6$\pm$6.1} & \underline{27.4$\pm$4.7} & 18.2$\pm$1.6 \\
walk-m-r & e & 15.8$\pm$7.2 & 15.7$\pm$4.5 & 7.1$\pm$4.6 & 15.2$\pm$5.3 & \underline{29.2$\pm$7.4} & \textbf{31.3$\pm$5.3} & 21.5$\pm$2.9 & 15.6$\pm$3.0 \\
walk-m-e & m & 39.9$\pm$13.1 & \underline{41.6$\pm$13.0} & 32.3$\pm$7.2 & 33.8$\pm$3.1 & 39.9$\pm$3.1 & 30.2$\pm$9.8 & 33.0$\pm$2.1 & \textbf{43.2$\pm$2.6} \\
walk-m-e & m-e & 49.1$\pm$6.9 & 45.8$\pm$9.4 & 40.1$\pm$4.5 & 44.7$\pm$2.9 & \underline{52.8$\pm$7.5} & \textbf{53.3$\pm$7.1} & 28.6$\pm$2.2 & 45.6$\pm$1.7 \\
walk-m-e & e & 40.4$\pm$11.9 & 56.4$\pm$3.5 & 43.7$\pm$4.4 & 45.3$\pm$10.4 & 52.9$\pm$10.1 & \underline{61.1$\pm$3.4} & 31.6$\pm$4.9 & \textbf{65.2$\pm$10.7} \\ \midrule
ant-m & m & 10.2$\pm$1.8 & 9.4$\pm$0.9 & 12.4$\pm$2.0 & 11.3$\pm$1.3 & 11.7$\pm$2.6 & 45.1$\pm$12.4 & \textbf{57.0$\pm$6.1} & \underline{56.4$\pm$5.1} \\
ant-m & m-e & 9.4$\pm$1.2 & 10.0$\pm$0.9 & 11.6$\pm$1.3 & 9.4$\pm$1.4 & 10.7$\pm$2.3 & 33.9$\pm$5.4 & \textbf{52.9$\pm$5.5} & \underline{50.7$\pm$9.1} \\
ant-m & e & 10.2$\pm$0.3 & 9.8$\pm$0.6 & 11.8$\pm$0.4 & 9.7$\pm$1.6 & 11.2$\pm$3.8 & 33.2$\pm$9.0 & \underline{54.1$\pm$1.9} & \textbf{55.7$\pm$6.4} \\
ant-m-r & m & 18.9$\pm$2.6 & 21.7$\pm$2.1 & 13.9$\pm$1.5 & 19.6$\pm$1.0 & 19.1$\pm$3.5 & 29.6$\pm$10.7 & \underline{53.3$\pm$12.0} & \textbf{57.2$\pm$8.2} \\
ant-m-r & m-e & 19.1$\pm$3.0 & 18.3$\pm$2.1 & 15.9$\pm$2.7 & 20.3$\pm$1.6 & 18.3$\pm$4.1 & 25.4$\pm$2.1 & \underline{42.4$\pm$10.9} & \textbf{44.0$\pm$11.6} \\
ant-m-r & e & 18.5$\pm$0.9 & 20.0$\pm$1.3 & 14.5$\pm$1.7 & 18.8$\pm$2.1 & 18.5$\pm$2.6 & 24.5$\pm$2.8 & \textbf{50.1$\pm$7.9} & \underline{40.5$\pm$13.9} \\
ant-m-e & m & 9.8$\pm$2.4 & 8.1$\pm$1.8 & 8.1$\pm$3.0 & 8.9$\pm$1.5 & 10.1$\pm$3.5 & 18.6$\pm$11.9 & \textbf{58.0$\pm$11.7} & \underline{53.8$\pm$10.4} \\
ant-m-e & m-e & 9.0$\pm$0.8 & 6.4$\pm$1.4 & 6.2$\pm$1.5 & 7.2$\pm$2.9 & 9.3$\pm$2.1 & 34.0$\pm$9.4 & \textbf{46.0$\pm$10.3} & \underline{43.8$\pm$6.1} \\
ant-m-e & e & 9.1$\pm$2.6 & 10.4$\pm$2.9 & 4.2$\pm$3.9 & 9.2$\pm$1.5 & 9.8$\pm$1.8 & 23.2$\pm$2.9 & \textbf{58.0$\pm$3.7} & \underline{45.7$\pm$18.9} \\
\midrule
\multicolumn{2}{c|}{\textbf{Total Score}} & 825.0 & 851.0 & 763.2 & 803.2 & 910.0 & 1160.7 & \underline{1270.4} & \textbf{1473.9} \\
\bottomrule
\end{tabular}
}
\end{table}

\clearpage
\newpage

\subsection{Performance Comparison across Gravity Shifts}
\label{subsecapp:gravity_sweep}

To examine how this benefit of source reuse varies with the magnitude of the gravity shift, we extend the HalfCheetah evaluation across eight target gravity levels ranging from 0.7 to 4.0 times the standard value (Table~\ref{table:gravity_shift_comparison}). Source gravity remains unchanged, and the coverage coefficient is fixed at $\eta=0.2$.
We use \texttt{halfcheetah-medium-expert-v2} as source data and construct expert target datasets with approximately 5K transitions collected in each target environment.
All gravity variants use the standard D4RL HalfCheetah normalization references.
As in the main-text comparisons, TCE(Mix) reports the best mean score over $\lambda\in\{0.1,0.5,0.9\}$ for each environment.

TCE(Mix) generally outperforms TCE(Gen) under mild gravity changes of 0.7--1.3$\times$ the standard value, whereas TCE(Gen) achieves higher mean scores at all evaluated gravity levels of 1.5$\times$ the standard value or above.
Among the evaluated mixing configurations, stronger gravity changes favor smaller source fractions.
Together, these observations suggest that the benefit of additionally incorporating target-near source transitions diminishes as the gravity shift grows, favoring the use of source states to expand generation support rather than the inclusion of their transitions.
This change in the relative benefit of the two forms of source utilization is consistent with the source--generation trade-off in Theorem~\ref{thm:gap}.

\begin{table}[H]
\centering
\captionsetup{skip=4pt}
\caption{Performance comparison across HalfCheetah gravity shifts with approximately 5K target transitions. The best model is in \textbf{bold} and the second-best is \underline{underlined}.}
\label{table:gravity_shift_comparison}
\small
\setlength{\tabcolsep}{4pt}
\renewcommand{\arraystretch}{1.0}
\begin{tabular}{@{}c|cccccc|cc@{}}
\toprule
Gravity ($\times$) & IQL* & DARA & BOSA & IGDF & MOBODY & OTDF & TCE(Gen) & TCE(Mix) \\
\midrule
0.7 & 39.5$\pm$1.1 & 36.9$\pm$0.7 & 36.0$\pm$0.7 & 34.8$\pm$1.4 & 38.4$\pm$0.6 & 40.4$\pm$0.7 & \underline{44.3$\pm$4.1} & \textbf{46.3$\pm$0.2} \\
0.9 & 42.5$\pm$0.5 & 43.3$\pm$0.9 & 43.3$\pm$0.3 & 44.7$\pm$1.4 & \underline{45.5$\pm$1.9} & 42.0$\pm$0.2 & 41.9$\pm$1.9 & \textbf{49.4$\pm$1.8} \\
1.1 & 50.5$\pm$1.1 & 67.0$\pm$2.1 & 41.1$\pm$0.3 & 54.7$\pm$7.3 & 53.2$\pm$1.5 & 51.0$\pm$0.2 & \underline{90.8$\pm$1.0} & \textbf{95.9$\pm$1.5} \\
1.3 & 38.1$\pm$1.1 & 51.0$\pm$2.6 & 39.2$\pm$1.1 & 48.2$\pm$1.8 & 39.9$\pm$0.1 & 38.6$\pm$0.6 & \underline{56.8$\pm$0.5} & \textbf{58.0$\pm$0.0} \\
1.5 & 32.7$\pm$1.6 & 37.3$\pm$1.0 & 37.5$\pm$0.2 & 36.2$\pm$1.4 & 35.5$\pm$0.7 & 33.9$\pm$0.7 & \textbf{45.9$\pm$0.1} & \underline{45.6$\pm$0.1} \\
2.0 & 20.3$\pm$0.1 & 20.8$\pm$0.2 & 22.4$\pm$0.7 & 21.2$\pm$1.5 & 22.6$\pm$0.2 & 20.6$\pm$1.7 & \textbf{36.6$\pm$0.2} & \underline{30.2$\pm$0.8} \\
3.0 & 2.8$\pm$1.4 & 3.4$\pm$1.8 & 2.8$\pm$1.8 & 8.4$\pm$1.2 & 5.1$\pm$1.3 & 3.7$\pm$0.8 & \textbf{42.1$\pm$0.2} & \underline{37.8$\pm$0.4} \\
4.0 & 1.0$\pm$0.5 & 0.5$\pm$0.0 & 2.1$\pm$0.6 & 0.6$\pm$0.4 & 0.9$\pm$0.7 & 1.4$\pm$0.5 & \textbf{38.2$\pm$0.1} & \underline{22.0$\pm$1.0} \\
\bottomrule
\end{tabular}
\end{table}

\clearpage
\newpage

\subsection{Comparison with Additional Baselines}
\label{subsecapp:meta_comp}
We additionally compare TCE(Gen) against Meta-DT*~\citep{wang2024meta} and DmC~\citep{dmc}. Meta-DT* is a representative meta-offline RL approach adapted for our setting by training on the union of source and target datasets, while DmC is a recent diffusion-based method that generates additional source-domain data.
Table~\ref{table:perform_kinematic_MetaDTcomparison} shows that Meta-DT* scores below TCE(Gen) on all evaluated tasks and below IQL* on many tasks.
This comparison indicates that the pooled-data Meta-DT* adaptation is less effective than TCE(Gen) in this cross-domain setting.
Meanwhile, TCE(Gen) outperforms DmC on a majority of tasks and achieves a higher total score, although DmC remains stronger on several tasks. This pattern is consistent with the potential benefit of generating target-aligned transitions rather than additional source-domain data.
 We note that since the official code for DmC is unavailable, we report the performance directly from the original paper.

{
\begin{table*}[h]
\centering
\captionsetup{skip=4pt}
\caption{Performance comparison of TCE(Gen) with Meta-DT* and DmC under kinematic shifts.}
\small
\setlength{\tabcolsep}{4pt}
\renewcommand{\arraystretch}{1.0}
\begin{tabular}{@{}S{1.5cm}l|ccccc@{}}
\toprule
Src. & Tgt. & IQL* & Meta-DT* & DmC & OTDF & TCE(Gen) \\
\midrule
half-m   & m   & 12.3$\pm$1.2 & 13.4$\pm$3.4 & 38.5$\pm$1.4 & \underline{40.2$\pm$0.0} & \textbf{42.2$\pm$2.4} \\
half-m   & m-e & 10.8$\pm$1.9 & 8.5$\pm$0.6 & \underline{19.1$\pm$1.0} & 10.1$\pm$4.0 & \textbf{40.7$\pm$2.7} \\
half-m   & e   & 12.6$\pm$1.7 & 5.0$\pm$0.1 & \underline{13.1$\pm$0.8} & 8.7$\pm$2.0 & \textbf{21.3$\pm$0.8} \\
half-m-r & m   & 10.0$\pm$5.4 & 5.5$\pm$0.7 & 19.5$\pm$1.8 & \underline{37.8$\pm$2.1} & \textbf{42.0$\pm$0.2} \\
half-m-r & m-e & 6.5$\pm$3.1 & 7.5$\pm$1.1 & \underline{11.4$\pm$2.1} & 9.7$\pm$2.0 & \textbf{41.6$\pm$0.5} \\
half-m-r & e   & 13.6$\pm$1.4 & 6.4$\pm$2.8 & \underline{15.6$\pm$2.9} & 7.2$\pm$1.4 & \textbf{23.4$\pm$1.5} \\
half-m-e & m   & 21.8$\pm$6.5 & 4.7$\pm$0.2 & \underline{38.4$\pm$1.4} & 30.7$\pm$9.6 & \textbf{41.9$\pm$0.3} \\
half-m-e & m-e & 7.6$\pm$1.4 & 7.7$\pm$3.2 & \underline{24.1$\pm$4.6} & 10.9$\pm$4.2 & \textbf{40.0$\pm$2.5} \\
half-m-e & e   & 9.1$\pm$2.4 & 2.8$\pm$0.2 & \textbf{13.4$\pm$2.0} & 3.2$\pm$0.6 & \underline{11.2$\pm$1.2} \\ \midrule
hopp-m   & m   & 58.7$\pm$8.4 & 5.0$\pm$0.2 & \textbf{69.8$\pm$2.3} & 65.6$\pm$1.9 & \underline{66.0$\pm$0.4} \\
hopp-m   & m-e & 68.5$\pm$12.4 & 4.9$\pm$0.2 & \textbf{78.2$\pm$5.1} & 55.4$\pm$25.1 & \underline{75.7$\pm$2.2} \\
hopp-m   & e   & \underline{79.9$\pm$35.5} & 5.2$\pm$0.2 & 59.8$\pm$21.8 & 35.0$\pm$19.4 & \textbf{91.5$\pm$5.7} \\
hopp-m-r & m   & 36.0$\pm$0.1 & 36.4$\pm$0.2 & \underline{64.8$\pm$2.4} & 35.5$\pm$12.2 & \textbf{65.9$\pm$1.4} \\
hopp-m-r & m-e & 36.1$\pm$0.1 & 36.3$\pm$0.1 & \underline{69.7$\pm$7.5} & 47.5$\pm$14.6 & \textbf{82.8$\pm$4.3} \\
hopp-m-r & e   & 36.0$\pm$0.1 & 37.6$\pm$0.1 & \underline{69.9$\pm$18.0} & 49.9$\pm$30.5 & \textbf{84.8$\pm$7.0} \\
hopp-m-e & m   & 66.0$\pm$0.5 & 3.5$\pm$0.4 & \textbf{69.6$\pm$1.3} & 65.3$\pm$2.4 & \underline{67.3$\pm$0.8} \\
hopp-m-e & m-e & 45.1$\pm$15.7 & 8.5$\pm$2.3 & \textbf{75.5$\pm$9.6} & 38.6$\pm$15.9 & \underline{74.7$\pm$8.9} \\
hopp-m-e & e   & 44.9$\pm$19.8 & 6.4$\pm$0.2 & \underline{64.5$\pm$24.2} & 29.9$\pm$11.3 & \textbf{65.0$\pm$7.5} \\ \midrule
walk-m   & m   & 34.3$\pm$9.8 & 5.0$\pm$0.3 & \textbf{63.2$\pm$4.2} & 49.6$\pm$18.0 & \underline{55.3$\pm$2.0} \\
walk-m   & m-e & 30.2$\pm$12.5 & 15.7$\pm$2.0 & \textbf{53.5$\pm$7.0} & \underline{43.5$\pm$16.4} & 35.9$\pm$10.2 \\
walk-m   & e   & 56.4$\pm$18.2 & 10.0$\pm$0.9 & \textbf{70.5$\pm$12.0} & 46.7$\pm$13.6 & \underline{62.8$\pm$1.7} \\
walk-m-r & m   & 11.5$\pm$7.1 & 3.4$\pm$1.0 & \underline{52.9$\pm$8.4} & 49.7$\pm$9.7 & \textbf{55.6$\pm$9.5} \\
walk-m-r & m-e & 9.7$\pm$3.8 & 14.6$\pm$0.1 & 36.4$\pm$5.4 & \textbf{55.9$\pm$17.1} & \underline{40.9$\pm$10.0} \\
walk-m-r & e   & 7.7$\pm$4.8 & 8.9$\pm$1.0 & 44.4$\pm$8.5 & \textbf{51.9$\pm$7.9} & \underline{49.9$\pm$7.4} \\
walk-m-e & m   & 41.8$\pm$8.8 & 8.5$\pm$0.8 & \textbf{59.4$\pm$6.8} & 44.6$\pm$6.0 & \underline{46.3$\pm$5.4} \\
walk-m-e & m-e & 22.2$\pm$8.7 & 10.2$\pm$8.7 & \textbf{53.2$\pm$7.3} & 16.5$\pm$7.2 & \underline{33.6$\pm$4.3} \\
walk-m-e & e   & 26.3$\pm$10.4 & 5.7$\pm$2.6 & \textbf{69.2$\pm$7.0} & 42.4$\pm$9.1 & \underline{54.6$\pm$4.4} \\ \midrule
ant-m    & m   & 50.0$\pm$5.6 & 13.9$\pm$0.7 & \textbf{62.1$\pm$0.6} & \underline{55.4$\pm$0.0} & 55.0$\pm$3.7 \\
ant-m    & m-e & 57.8$\pm$7.2 & 14.8$\pm$0.3 & \textbf{68.9$\pm$1.0} & 60.7$\pm$3.6 & \underline{65.1$\pm$6.4} \\
ant-m    & e   & 59.6$\pm$18.5 & 14.9$\pm$0.1 & \underline{92.1$\pm$3.5} & 90.4$\pm$4.8 & \textbf{93.9$\pm$4.0} \\
ant-m-r  & m   & 43.7$\pm$4.6 & 22.3$\pm$0.2 & \textbf{61.9$\pm$0.5} & 52.8$\pm$4.4 & \underline{55.6$\pm$9.1} \\
ant-m-r  & m-e & 36.5$\pm$5.9 & 21.0$\pm$0.1 & \underline{58.8$\pm$3.6} & 54.2$\pm$5.2 & \textbf{65.6$\pm$5.4} \\
ant-m-r  & e   & 24.4$\pm$4.8 & 26.8$\pm$2.2 & 43.8$\pm$2.6 & \underline{74.7$\pm$10.5} & \textbf{92.3$\pm$3.7} \\
ant-m-e  & m   & 49.5$\pm$4.1 & 13.5$\pm$0.1 & \textbf{60.6$\pm$1.3} & 50.2$\pm$4.3 & \underline{57.2$\pm$2.7} \\
ant-m-e  & m-e & 37.2$\pm$2.0 & 13.6$\pm$0.1 & \underline{60.4$\pm$3.7} & 48.8$\pm$2.7 & \textbf{69.0$\pm$1.8} \\
ant-m-e  & e   & 18.7$\pm$8.1 & 18.1$\pm$1.2 & 76.0$\pm$4.1 & \underline{78.4$\pm$12.2} & \textbf{97.3$\pm$1.3} \\
\midrule
\multicolumn{2}{c|}{\textbf{Total Score}} & 1193.0 & 446.2 & \underline{1902.2} & 1547.6 & \textbf{2063.9} \\
\bottomrule
\end{tabular}
\label{table:perform_kinematic_MetaDTcomparison}
\end{table*}
}

\clearpage
\newpage

\section{Sample Fidelity Analysis}
\label{secapp:analy_ymax}

{ This section complements the main coverage--error analysis in Section~\ref{subsec:analysis} with quantitative fidelity trends on the HalfCheetah and Hopper morphology-shift (\texttt{medium-to-expert}) tasks and a qualitative target-simulator visualization.}

{
\subsection{Quantitative Error Analysis}

Coverage expansion increases generation error to different degrees across environments. We analyze the fidelity of generated samples on the HalfCheetah and Hopper morphology-shift tasks (\texttt{medium-to-expert}) with respect to the coverage coefficient $\eta$. Following the definitions in Section~\ref{subsec:analysis}, we report three metrics: action reconstruction error measured using a reference target inverse-dynamics model, and reward prediction and transition dynamics errors measured using the true target simulator.

\paragraph{HalfCheetah.}
HalfCheetah exhibits increasing generation error as coverage expands. Fig.~\ref{fig:cov_halfcheetah}(a)--(c) show that action, reward, and transition errors increase as $\eta$ expands state coverage. Together with the main performance analysis, this trend supports balancing broader coverage against increased generation error rather than minimizing error alone.

\begin{figure}[H]
    \begin{center}
    \centerline{\includegraphics[width=0.8\columnwidth]{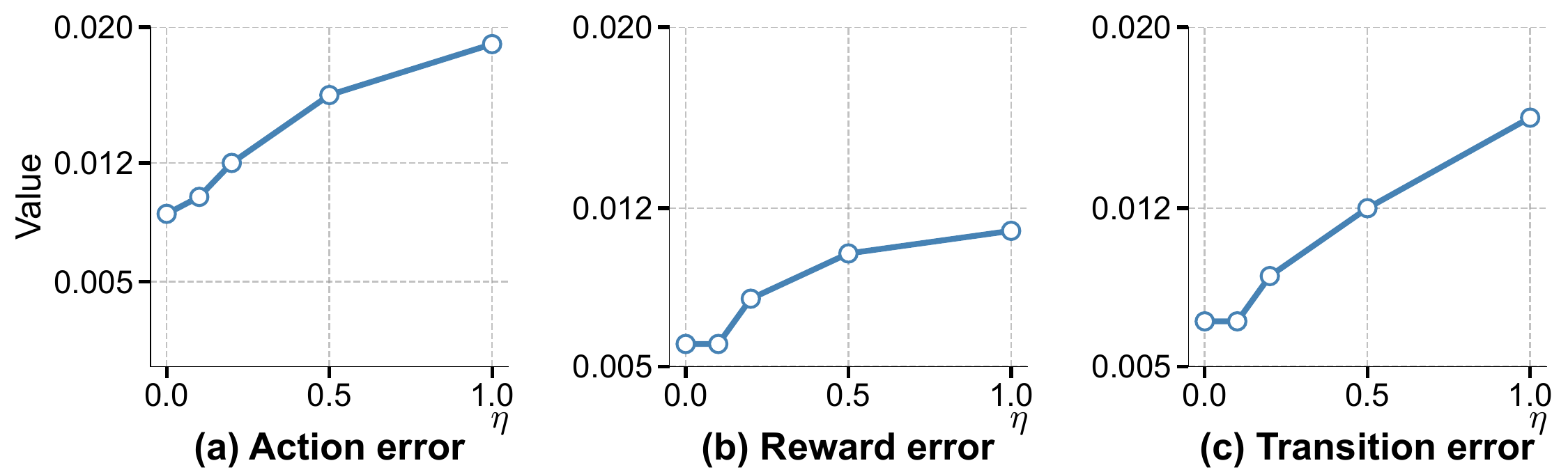}}
    \caption{Sample-fidelity errors with respect to $\eta$ in HalfCheetah morphology shifts.}
    \label{fig:cov_halfcheetah}
    \end{center}
\end{figure}

\paragraph{Hopper.}
Hopper exhibits a sharp increase in action error under aggressive coverage expansion. In particular, Fig.~\ref{fig:ymax_outlier_hopper}(a) shows a sharp rise in action reconstruction error at larger $\eta$ values, accompanied by increases in reward and transition errors in Fig.~\ref{fig:ymax_outlier_hopper}(b)--(c). This sensitivity highlights the need to limit expansion when generation error grows rapidly, rather than applying the same expansion level across environments.

\begin{figure}[H]
    \begin{center}
    \centerline{\includegraphics[width=0.8\columnwidth]{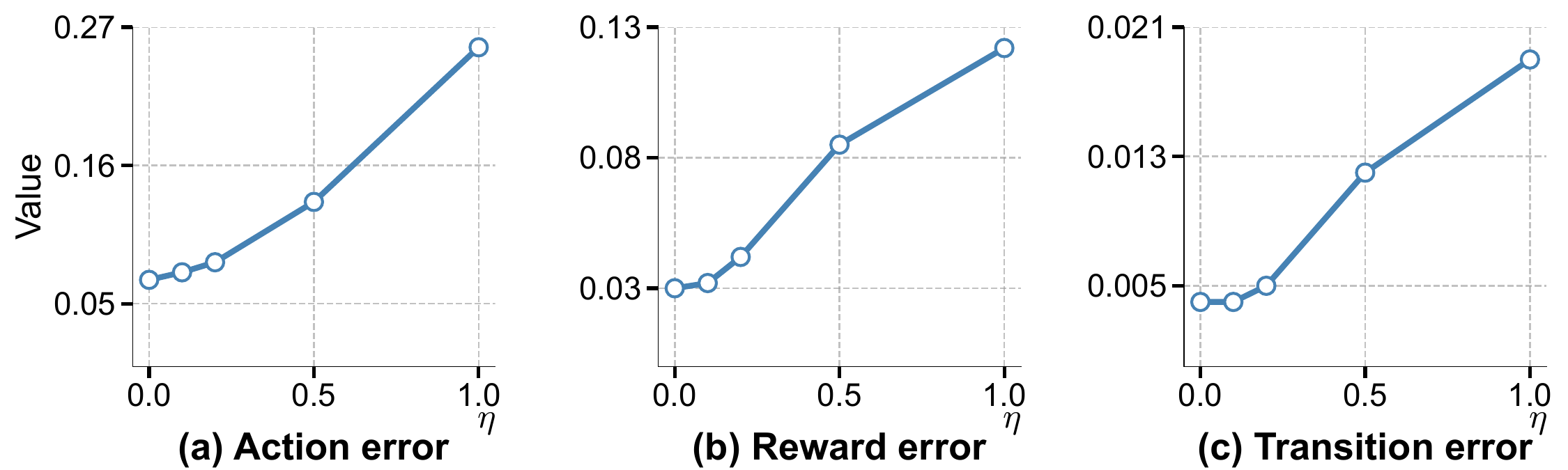}}
    \caption{Sample-fidelity errors with respect to $\eta$ in Hopper morphology shifts.}
    \label{fig:ymax_outlier_hopper}
    \end{center}
\end{figure}
}

{
\subsection{Qualitative Simulator Visualization}
\label{secapp:render}

In the two HalfCheetah morphology-shift examples, TCE's generated next states closely match the target-simulator outcomes. For each target-near source state $s_t$, we project the state into the target simulator by initializing it with the same state variables. TCE generates $\hat{s}_{t+1}$ and infers the corresponding action $\hat{a}_t$; we then execute this action for one target-environment step and visually compare the simulator rollout with the generated next state. The agreement in Fig.~\ref{fig:target_simulator_check} provides qualitative evidence of one-step target-dynamics consistency for these examples, complementing the simulator-based transition error in Table~\ref{fig:error_table}. The target simulator is used only for this evaluation.

\begin{figure}[H]
    \centering
    \includegraphics[width=0.9\textwidth]{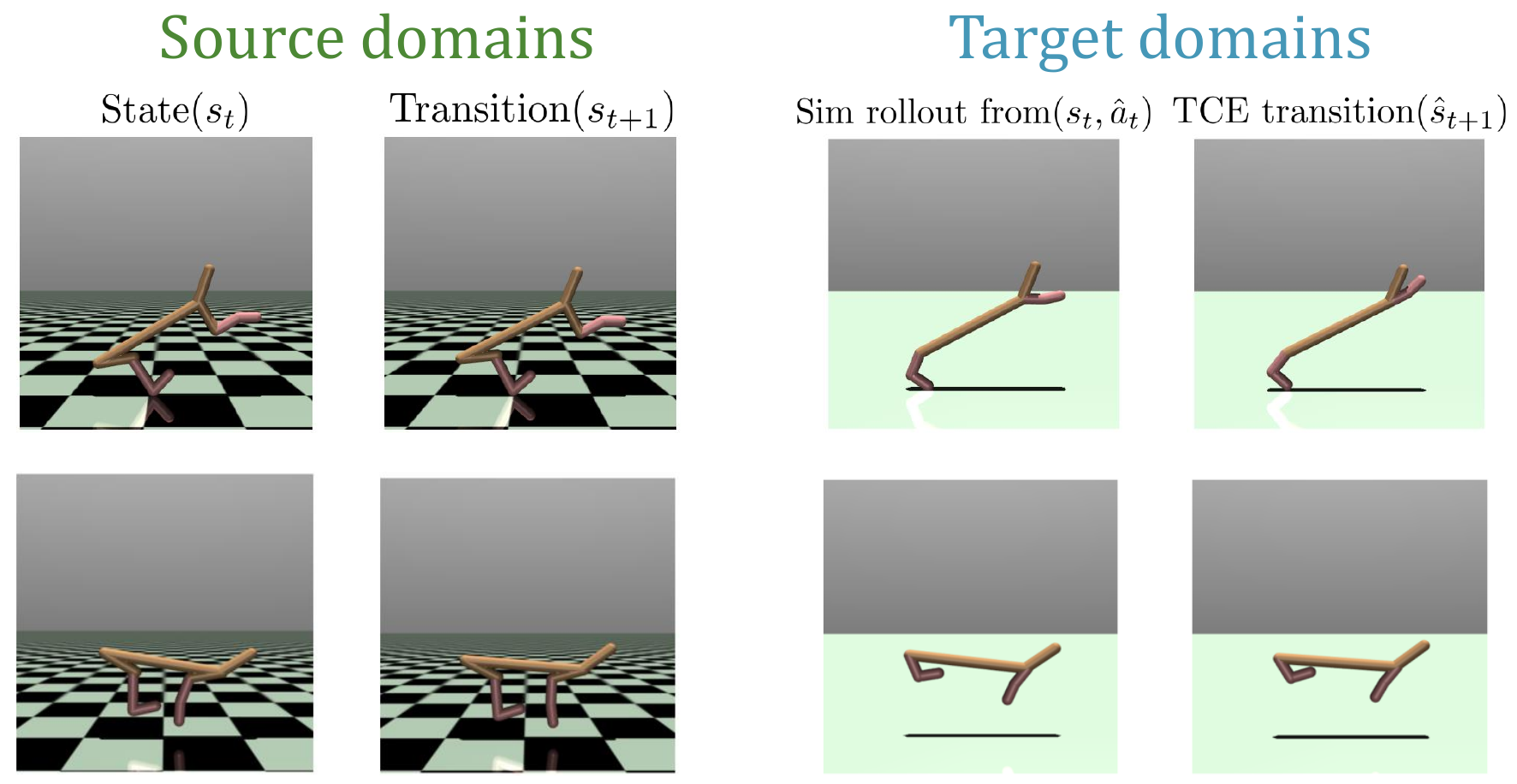}
    \caption{Qualitative comparison on the HalfCheetah morphology shift. Each row shows a target-near source state $s_t$ and its source-domain next state $s_{t+1}$, followed by the target-simulator rollout obtained after projecting $s_t$ into the target domain and executing $\hat{a}_t$, and TCE's generated next state $\hat{s}_{t+1}$.}
    \label{fig:target_simulator_check}
\end{figure}
}

\clearpage
\newpage

\section{Computational Complexity}
\label{secapp:compu}

{
TCE(Gen) adds generation-related computation while maintaining peak GPU memory usage comparable to the baselines. This section compares the computational resources required by IQL* \citep{iql}, OTDF \citep{lyu2025cross}, and TCE(Gen) over the 36 morphology-shift tasks. Experiments were conducted on a server with AMD EPYC 7513 CPUs and NVIDIA RTX 3090 GPUs. While all methods require approximately three hours for the offline RL phase, OTDF adds roughly 18 minutes for data selection. TCE(Gen) incurs approximately 48 minutes for score network training and 16 minutes for sampling, averaged over the evaluated tasks. These phases and offline RL run sequentially with intermediate memory deallocation, keeping the reported peak GPU memory usage at roughly 2 GB. Table~\ref{tab:runtime_comparison_all} therefore shows additional preprocessing time with comparable peak memory usage, in exchange for the performance improvements reported in the main text.

\begin{table}[H]
\centering
\captionsetup{skip=4pt}
\caption{Comprehensive runtime and memory usage comparison averaged over 36 morphology-shift tasks. Note that GPU memory usage represents the peak allocation during the entire pipeline.}
\label{tab:runtime_comparison_all}
\small
\setlength{\tabcolsep}{4pt}
\renewcommand{\arraystretch}{1.0}
\begin{tabular}{@{}l|cccc@{}}
\toprule
Method & Model Training & Data Sampling & Offline RL & Peak GPU Memory \\ 
\midrule
IQL* & -- & -- & 3~h & 2 GB \\
OTDF & -- & 18~min & 3~h & 2 GB \\
TCE(Gen) & 48~min & 16~min & 3~h & 2 GB \\
\bottomrule
\end{tabular}
\end{table}
}

\section{Additional Analysis}
\label{secapp:addabl}

{ This section presents detailed analyses not covered in the main text. Appendix~\ref{subsecapp:eta_lambda_abl} examines coverage expansion and source--generation mixing.  Appendix~\ref{subsecapp:target-size} analyzes the effect of target-data size on inverse-dynamics error and policy performance. Appendix~\ref{subsecapp:deno-steps} analyzes sensitivity to the denoising step $K$.}

\subsection{Ablation Studies on Coverage Expansion and Source--Generation Mixing}
\label{subsecapp:eta_lambda_abl}

\paragraph{Coverage Coefficient $\eta$.} 
Moderate coverage expansion generally improves performance over Simple Aug., while larger $\eta$ can reduce the gain.
Fig.~\ref{fig:app_eta} illustrates the performance trends across different values of $\eta$  for TCE(Gen), with $\lambda=0$ fixed, under morphology, kinematic, and gravity shifts on the Ant task (\texttt{medium-replay-to-medium-expert}). As shown, the performance remains relatively stable across a wide range, generally outperforming the simple augmentation baseline. However, we observe a slight performance degradation as $\eta$ increases to $0.5$. This is consistent with the coverage--error trade-off in Section~\ref{subsec:analysis}: expanding support can improve learning, but excessive expansion increases generation error.
Accordingly, the choice of $\eta$ should account for both coverage gains and the generation-error sensitivity of each setting.

\begin{figure}[h]
    \centering
    \includegraphics[width=0.9\textwidth]{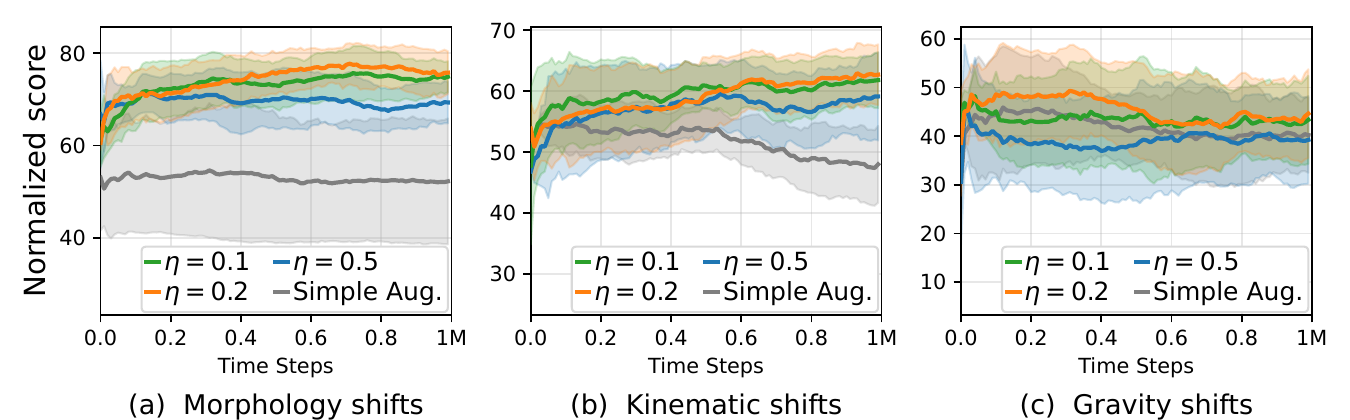}
    \caption{Performance sensitivity to $\eta$ across different domain shifts on the Ant \texttt{medium-replay-to-medium-expert} task.}
    \label{fig:app_eta}
\end{figure}

\paragraph{Mixture Coefficient $\lambda$.}
The benefit of source reuse depends on the dynamics gap of selected source transitions relative to generation error. Fig.~\ref{fig:app_lambda_mix} analyzes sensitivity to the source--generation mixture by varying $\lambda$ with $\eta=0.2$ fixed, with $\lambda=0$ corresponding to TCE(Gen) and $\lambda>0$ to TCE(Mix). Increasing $\lambda$ replaces generated transitions with NN-selected real source transitions at a fixed augmentation budget.  As suggested by Theorem~\ref{thm:gap}, source reuse can complement generation when the selected transitions are sufficiently target-aligned, whereas a larger dynamics gap can favor generation.

\begin{figure}[h]
    \centering
    \includegraphics[width=0.9\textwidth]{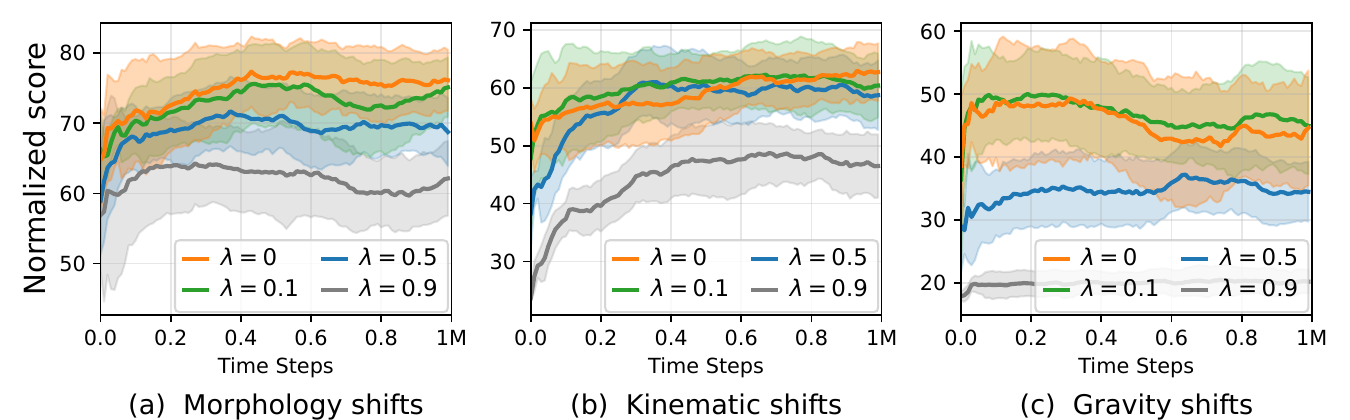}
    \caption{Sensitivity to the mixture coefficient $\lambda$ across different domain shifts on the Ant \texttt{medium-replay-to-medium-expert} task.}
    \label{fig:app_lambda_mix}
\end{figure}

\vspace{-0.1in}
\subsection{Effect of Target Data Size on Inverse Dynamics Error}
\label{subsecapp:target-size}

Increasing target-data size reduces inverse-dynamics error and improves policy performance, with smaller score gains beyond 5K transitions. We vary the target-data size from 2K to 20K transitions and report the resulting action prediction error and normalized score in Table~\ref{table:target_size}. Reducing the data from 5K to 2K increases action error from 0.065 to 0.142 and lowers the score from 94.0 to 90.3. Increasing the data to 20K further reduces the error to 0.015, while the score rises to 95.0, indicating diminishing performance gains despite continued error reduction.

\begin{table}[h]
\centering
\captionsetup{skip=4pt}
\caption{Action prediction error and normalized score under different target-data sizes.}
\small
\setlength{\tabcolsep}{4pt}
\renewcommand{\arraystretch}{1.0}
\begin{tabular}{@{}l|cccc@{}}
\toprule
Metric & 2K & 5K (ours) & 10K & 20K \\
\midrule
Action error & 0.142$\pm$0.018 & 0.065$\pm$0.012 & 0.042$\pm$0.009 & \textbf{0.015$\pm$0.006} \\
Score & 90.3$\pm$4.2 & 94.0$\pm$4.5 & 94.4$\pm$3.8 & \textbf{95.0$\pm$3.1} \\
\bottomrule
\end{tabular}
\label{table:target_size}
\end{table}

\vspace{-0.1in}
\subsection{Ablation Study on Denoising Step $K$}
\label{subsecapp:deno-steps}

TCE(Gen) shows modest changes in average performance across the tested denoising-step counts.The denoising-step hyperparameter $K$ governs the trade-off between sampling cost and generation quality through the discretization granularity of the reverse process. Table~\ref{table:perform_ant_k} presents results on the Ant morphology-shift environment for $K \in \{100, 200, 500\}$, with $K=500$ as the default. While $K=500$ achieves the highest total score, $K=100$ and $K=200$ remain competitive on average, although individual tasks differ in sensitivity. 

\begin{table*}[ht]
\centering
\captionsetup{skip=4pt}
\caption{{ Performance of TCE(Gen) on Ant morphology-shift tasks at different $K$ values.}}
\small
\setlength{\tabcolsep}{4pt}
\renewcommand{\arraystretch}{1.0}
\begin{tabular}{@{}ll|ccc@{}}
\toprule
Src. & Tgt. & $K=100$ & $K=200$ & $K=500$ \\
\midrule
ant-m   & m     & \textbf{42.2$\pm$1.8} & \underline{42.0$\pm$0.3} & 41.5$\pm$0.8 \\
ant-m & m-e   & 72.5$\pm$2.3 & \textbf{74.1$\pm$0.5} & \underline{72.7$\pm$2.5} \\
ant-m & e     & 91.9$\pm$0.1 & \underline{92.8$\pm$2.8} & \textbf{94.0$\pm$4.5} \\
ant-m-r & m     & 40.7$\pm$1.2 & \underline{40.8$\pm$0.7} & \textbf{41.5$\pm$1.5} \\
ant-m-r & m-e   & \underline{70.0$\pm$2.8} & 69.2$\pm$3.1 & \textbf{71.0$\pm$8.4} \\
ant-m-r & e     & 90.8$\pm$1.6 & \underline{91.3$\pm$3.8} & \textbf{96.1$\pm$0.3} \\
ant-m-e & m     & \underline{41.4$\pm$0.9} & 41.2$\pm$0.6& \textbf{41.9$\pm$0.6} \\
ant-m-e & m-e   & \textbf{66.5$\pm$3.1} & 66.1$\pm$1.6& \underline{66.4$\pm$1.2} \\
ant-m-e & e     & 93.7$\pm$0.5 & \underline{94.0$\pm$1.1} & \textbf{96.6$\pm$0.5} \\
\midrule
\multicolumn{2}{c|}{\textbf{Total Score}} & 609.7 & \underline{611.5} & \textbf{621.7} \\
\bottomrule
\end{tabular}
\label{table:perform_ant_k}
\end{table*}

\end{document}